%% file: main.tex
\documentclass[11pt, UTF8]{article}
\usepackage[margin=1in]{geometry}
\usepackage{amsmath,amssymb,amsthm,mathtools}
\usepackage{booktabs,longtable,tabularx,array,multirow}
\usepackage{multicol,capt-of}
\usepackage{enumitem}
\usepackage{microtype}
\usepackage{xcolor}
\usepackage[hidelinks]{hyperref}
\usepackage{cleveref}
\usepackage{fancyhdr}
\usepackage{lastpage}
\usepackage{listings}

\usepackage[normalem]{ulem}

\setlist[itemize]{leftmargin=1.5em,itemsep=0.25em,topsep=0.3em}
\setlist[enumerate]{leftmargin=1.7em,itemsep=0.25em,topsep=0.3em}

\newtheorem{theorem}{Theorem}[section]
\newtheorem{proposition}[theorem]{Proposition}
\newtheorem{lemma}[theorem]{Lemma}

\theoremstyle{definition}
\newtheorem{definition}[theorem]{Definition}
\newtheorem{remark}[theorem]{Remark}

\newcommand{\R}{\mathbb{R}}
\newcommand{\Z}{\mathbb{Z}}
\newcommand{\N}{\mathbb{N}}
\newcommand{\MAX}{\operatorname{MAX}}

\newcommand{\A}{\mathcal{A}}

\title{Representing MAX functions using two-hidden-layer ReLU networks}
\author{Zhimao Wang \thanks{Department of Applied Mathematics, Johns Hopkins University, {\tt zwang539@jh.edu}, {\tt basu.amitabh@jhu.edu}}  \and Amitabh Basu\footnotemark[1]}
\date{}

\begin{document}
\maketitle

\begin{abstract}
We study exact representations of
\[
  \MAX_N(x)=\max\{x_1,\ldots,x_N\}
\]
using two-hidden-layer ReLU neural networks. This problem has been studied by several researchers in recent years in an attempt to characterize the exact number of hidden layers required to represent continuous piecewise linear functions. The best lower bound is 2, while the current upper bound is logarithmic in $N$. It remains completely open if the right answer is a constant number of hidden layers (possibly even 2!) or not. In fact,  a recent breakthrough was the representation of \(\MAX_5\) as a two-hidden-layer ReLU function which was obtained in~\cite{bakaev2026better}, and the case of \(\MAX_N\) was stated as open for $N \geq 6$ in that paper. 

Using a careful computer assisted search, we are able to obtain two-hidden-layer ReLU representations of \(\MAX_5\), \(\MAX_6\), \(\MAX_7\), and \(\MAX_8\). We obtain these by considering rational linear combination of terms of the form
\[
  \max\!\left\{
    \sum_{r=1}^{s}\max(x_{a_r},x_{b_r}),
    \sum_{r=1}^{s}\max(x_{c_r},x_{d_r})
  \right\},
\]
where $a_r, b_r, c_r, d_r \in \{1, \ldots, N\}$. Each inner maximum of two coordinates can be computed in a first hidden layer, and the outer maximum of the two side-sums can be computed in a second hidden layer.  Consequently, every finite linear combination of these terms has a two-hidden-layer ReLU realization.  An identity for \(\MAX_N\) in this form therefore gives  an exact two-hidden-layer ReLU representation of \(\MAX_N\).

 Very recently, two-hidden-layer representations of \(\MAX_N\) of the above form were obtained for all $N \leq 10$ in~\cite{ruess2026shallower}. Our representations are different and were developed independently. While our techniques share most of the high level ideas presented in~\cite{ruess2026shallower}, there are also some minor differences which may be of interest for future research on this problem.

\end{abstract}



\section{Introduction}

The {\em ReLU activation function} is given by
\[
  \rho(t):=\max\{0,t\}.
\]
A {\em (feedforward) ReLU neural network function} is an alternating composition of affine linear maps and coordinatewise application of $\rho$ such that the overall composition begins and ends with affine maps. The number of $\rho$ applications is termed the number of {\em hidden layers} of the ReLU neural network. A central question that has been studied for over a decade now is the question of expressivity of such neural networks, i.e., what family of functions can be represented by such neural networks and what bounds can one provide on the required number of hidden layers and the total number of neurons to such functions. Since $\rho$ is a continuous, piecewise linear function, any ReLU neural network function is continuous and piecewise linear. It was observed in~\cite{arora2018understanding} that every continuous, piecewise linear function on $\R^n$ is a ReLU neural network function with at most $\left\lceil\log_2(n+1)\right\rceil$ hidden layers. It has been a long standing question if $\left\lceil\log_2(n+1)\right\rceil$ is required in the worst case, i.e., there exist some continuous piecewise linear functions that cannot be represented with strictly smaller number of hidden layers. While this was shown to be true under integer weight restrictions in~\cite{haase2023lower} and a lower bound of $\lceil \log_3(n)\rceil$ was shown under decimal-fraction weight restrictions~\cite{averkov2025expressiveness}. Conditional super-constant lower bounds were established in~\cite{grillo2026depth}. A breakthrough was obtained in~\cite{bakaev2026better} where the authors show that $\lceil \log_3(n-1) \rceil + 1$ hidden layers are sufficient if one allows arbitrary real weights. However, the best known general lower bound remains 2; no continuous piecewise linear function is known that requires more than 2 hidden layers. 

It has been known for a while~\cite{hertrich2023towards} that it is sufficient to focus attention on the function
\[
  \MAX_N(x)=\max\{x_1,\ldots,x_N\}.
\]
This is because every continuous piecewise linear function is a linear combination of functions of the form $\MAX_{n+1}\circ\phi$, where $\phi: \R^n \to \R^{n+1}$ is an affine linear function~\cite{wang2005generalization}. Therefore, the number of hidden layers needed to represent $\MAX_N$ gives us the number of hidden layers to represent arbitrary continuous piecewise linear functions in $\R^n$ by setting $N=n+1$. A recent paper~\cite{ruess2026shallower} obtains two-hidden-layer representations for $\MAX_N$ for all $N \leq 10$. As a corollary, the authors establish the best known upper bound of $\left\lceil \log_5\left(\frac{n+1}{2}\right) \right\rceil + 1$ for ReLU neural network representations of contiuous piecewise linear functions on $\R^n$.

The paper~\cite{ruess2026shallower} obtains its representations of $\MAX_N$, $N \leq 10$ using a linear combination of terms of the form

\begin{equation}\label{eq:intro-atom}
  \max\!\left\{
    \sum_{i=1}^{k}\max(x_{a_i},x_{b_i}),
    \sum_{i=1}^{k}\max(x_{c_i},x_{d_i})
  \right\}.
\end{equation}
We will call such terms as {\em atoms} in this paper. For fixed $k \in \N$, we will call them {\em atoms of degree $k$.} The inner pair maxima can be evaluated in a first hidden ReLU layer, while the outer maximum of the
two side-sums can be evaluated in a second hidden layer.  Thus every finite signed linear combination
of atoms of the form \eqref{eq:intro-atom} has a two-hidden-layer ReLU realization.  The same atom also
has a convex-geometric interpretation: each sum is the support function of a zonotope, and the
outer maximum is the support function of the convex hull of the two zonotopes.  An identity for
$\MAX_N$ in this ansatz therefore gives both an exact network representation and, after
clearing denominators, an equality of Minkowski sums of polytopes.

The main challenge is proving that a large signed combination of such
atoms agrees with $\MAX_N$ on all of $\mathbb R^N$. An important observation that was made in~\cite{ruess2026shallower} is that $\MAX_N$ is invariant with respect to permutations of coordinates. Moreover any particular atom of the form~\eqref{eq:intro-atom} can be transformed into another such atom under permutations of coordinates. More formally, one can consider the group action of the symmetric group $\Sigma_N$ on the set of atoms of the form~\eqref{eq:intro-atom}, for a fixed $k$: For any $\sigma\in \Sigma_N$ and an atom $\mathcal{A}(x) = \max\!\left\{
    \sum_{i=1}^{k}\max(x_{a_i},x_{b_i}),
    \sum_{i=1}^{k}\max(x_{c_i},x_{d_i})
  \right\}$, define $\sigma \mathcal{A}$ as the atom $\max\!\left\{
    \sum_{i=1}^{k}\max(x_{\sigma(a_i)},x_{\sigma(b_i)}),
    \sum_{i=1}^{k}\max(x_{\sigma(c_i)},x_{\sigma(d_i)})
  \right\}$.


Thus, any linear combination of atoms of the form~\eqref{eq:intro-atom} such that all atoms in the same orbit under this group action have the same coefficient would give rise to a function that is invariant with respect to coordinate permutations. Consequently, one can try to find such a linear combination that equals $\MAX_N$ on the single sorted chamber
 \[
   \mathcal{C}:=\{x \in \R^N: x_1\geq x_2\geq\cdots\geq x_N\}.
 \]
If equality holds on this sorted chamber, we must have equality over all of $\R^N$ by symmetry with respect to permutations; see the discussion in Section 3 and 4 of~\cite{ruess2026shallower}. This symmetry helps to reduce the search for such linear combinations. We now formalize this approach.


\begin{definition}\label{def:atom-pattern} Fix any number $k\geq 1$. Define an {\em atom pattern of degree $k$} to be a list
\[
 P=[p_1,\ldots,p_k \mid q_1,\ldots,q_k],
\]
where $p_1, \ldots, p_k, q_1, \ldots, q_k \in \{1, \ldots, N\} \times \{1, \ldots, N\}$ be ordered pairs (with repetition allowed) of numbers in $\{1, \ldots, N\}$. 
Define the corresponding functions\[
 L_P(x)=\sum_{i=1}^{k}m_{p_i}(x),
 \qquad
 R_P(x)=\sum_{i=1}^{k}m_{q_i}(x),
\]
where $m_{(a,b)}(x)=\max\{x_a,x_b\}$. The {\em atom represented by $P$ is then defined as}
\[
 \A_P(x)=\max\{L_P(x),R_P(x)\}
\]
\medskip

For any permutation $\sigma\in \Sigma_N$ and any ordered pair $(a,b) \in \{1, \ldots, N\} \times \{1, \ldots, N\}$, define $\sigma((a,b))= (\sigma(a), \sigma(b))$. This defines a group action of the symmetric group $\Sigma_N$ on the set of all atom patterns of degree $k$ by defining $$\sigma P = [\sigma(p_1), \ldots, \sigma(p_k)\mid \sigma(q_1), \ldots, \sigma(q_k)].$$

\end{definition}

\begin{lemma}\label{lem:atom-max-linear} Fix a natural number $k\geq 1$. For any atom pattern $P$ of degree $k$, there exist $\eta_L, \eta_R \in \Z^N_+$ such that $\|\eta_L\|_1 = \|\eta_R\|_1 = k$ such that $$\mathcal{A}_P(x) = \max\{\langle \eta_L, x\rangle, \langle \eta_R, x\rangle\}$$ for all $x \in \mathcal{C}$. $\eta_L$ and $\eta_R$ are given by the closed-form formulas: $$(\eta_L)_i = \sum_{j=1}^k \mathbf{1}_{\min(p_j) = i} \qquad \forall i = 1, \ldots, N,$$
$$(\eta_R)_i = \sum_{j=1}^k \mathbf{1}_{\min(q_j) = i} \qquad \forall i = 1, \ldots, N,$$ where $\min(p_j)$ and $\min(q_j)$ denote the minimum of the pair of numbers in $p_j$ and $q_j$, respectively.
\end{lemma}

\begin{proof}
Write $p_j=(a_j,b_j)$ and $q_j=(c_j,d_j)$. Since $x\in\mathcal C$ has ordered coordinates,
\[
\max\{x_{a_j},x_{b_j}\}=x_{\min\{a_j,b_j\}}=x_{\min(p_j)},
\]
and similarly $\max\{x_{c_j},x_{d_j}\}=x_{\min(q_j)}$. Hence
\[
\sum_{j=1}^k \max\{x_{a_j},x_{b_j}\}
=\sum_{j=1}^k x_{\min(p_j)}
=\sum_{j=1}^k \sum_{i=1}^N \mathbf{1}_{\min(p_j)=i} x_i
=\sum_{i=1}^N\left(\sum_{j=1}^k \mathbf{ 1}_{\min(p_j)=i}\right)x_i
=\langle\eta_L,x\rangle,
\]
and analogously the right side-sum equals $\langle\eta_R,x\rangle$. Taking the outer maximum gives
\[
\mathcal A_P(x)=\max\{\langle\eta_L,x\rangle,\langle\eta_R,x\rangle\}.
\]
Finally, for each $j$, exactly one indicator $\mathbf{1}_{\min(p_j)=i}$ is equal to 1 as $i$ ranges from $1$ to $N$.
Hence the sum of the entries of $\eta_L$ is $k$, and similarly for $\eta_R$.
Since these vectors have nonnegative entries, $\|\eta_L\|_1 = \|\eta_R\|_1 = k$.
\end{proof}

\subsection{The approach in Rueß et al~\cite{ruess2026shallower}} Consider the group action of $\Sigma_N$ on the set of atom patterns of degree $k$ from Definition~\ref{def:atom-pattern}. Let $\mathcal{O}$ be any orbit under this action. Rueß et al~\cite{ruess2026shallower} define the function $$F_{\mathcal{O}} := \sum_{P \in \mathcal{O}}\mathcal{A}_P$$ and aim to find linear coefficients $\lambda_\mathcal{O} \in \R$ for every orbit $\mathcal{O}$ such that $$\MAX_N = \sum_{\mathcal{O}}\lambda_{\mathcal{O}}F_{\mathcal{O}},$$ where the sum is over all orbits of atoms of degree $k$ under the action of the symmetric group $\Sigma_N$. For this purpose, they use Lemma~\ref{lem:atom-max-linear} and the fact that $$\max\{\langle \eta_L, x\rangle, \langle \eta_R, x\rangle\} = \langle \eta_L, x \rangle + \max\{0, \langle \eta_R - \eta_L, x \rangle\}$$ to obtain that \begin{equation}\label{eq:ruess-equation}\sum_{\mathcal{O}}\lambda_{\mathcal{O}}F_{\mathcal{O}} = \sum_{\ell \in \mathcal{L}}a_\ell \ell(x) + \sum_{d \in \mathcal{D}} b_d \max\{0,\langle d, x \rangle\}\end{equation} where $\mathcal{L}$ is an explicit family of linear functions derived from the atoms of degree $k$, $\mathcal{D}$ is an explicit family of vectors in $\R^N$ also derived from the atoms of degree $k$, $a_\ell \in \R$, $\ell\in \mathcal{L}$ and $b_d \in \R$, $d \in \mathcal{D}$. It should be further noted that the coefficients $a_\ell$, $\ell \in \mathcal{L}$ and $b_d$, $d \in \mathcal{D}$ are explicit linear combinations of the coefficients $\lambda_\mathcal{O}$. As mentioned above, $\MAX_N$ and $F_{\mathcal{O}}$ are invariant with respect to coordinate permutations and thus we may restrict our attention to the sorted chamber $\mathcal{C}$ where $\MAX_N(x) = x_1$. Thus, if we wish to express $x_1 = \MAX_N(x) = \sum_{\mathcal{O}}\lambda_{\mathcal{O}}F_{\mathcal{O}}$, a sufficient condition is that $$\sum_{\ell \in \mathcal{L}}a_\ell \ell(x) = x_1, \qquad b_d = 0 \quad \forall d \in \mathcal{D}.$$
which becomes an explicit linear system in the  coefficients $\lambda_{\mathcal{O}}$, that is solved to obtain the desired representation of $\MAX_N$.

\subsection{Our approach}\label{sec:our-approach} In this paper, we set up a similar linear system to express $\MAX_N$ as a linear combination of atoms of degree $k$, but with two important differences compared to Rueß et al~\cite{ruess2026shallower}:

\begin{enumerate}
\item We quotient out the set of atoms by certain symmetries before forming the sums that are invariant under coordinate permutations.
\item We do not use the equivalence $\max\{\langle \eta_L, x\rangle, \langle \eta_R, x\rangle\} = \langle \eta_L, x \rangle + \max\{0, \langle \eta_R - \eta_L, x \rangle\}$ and instead work directly with the terms $\max\{\langle \eta_L, x\rangle, \langle \eta_R, x\rangle\}$. This leads to an even more important difference in the linear system we set up. We elaborate on this more below after formalizing our approach.
\end{enumerate}

\paragraph{Group of symmetries for the atoms.} We consider the following transformations on the set of atom patterns of degree $k$.

\begin{itemize}
  \item An ordered pair $p_i = (a,b)$ or $q_i = (a,b)$ in an atom $P = [p_1, \ldots, p_k \mid q_1, \ldots, q_k]$ is replaced by its reversal $(b,a)$.
  \item permuting the $k$ terms within each side, i.e., an atom $[p_1, \ldots, p_k \mid q_1, \ldots, q_k]$ is mapped to the atom $[p_{\sigma_L(1)}, \ldots, p_{\sigma_L(k)} \mid q_{\sigma_R(1)}, \ldots, q_{\sigma_R(k)}]$ for some choice of permutations $\sigma_L, \sigma_R \in \Sigma_k$. 
  \item Switching the left and right sides, i.e., an atom $[p_1, \ldots, p_k \mid q_1, \ldots, q_k]$ is mapped to the atom $[q_1, \ldots, q_k \mid p_1, \ldots, p_k]$
\end{itemize}

Note that if an atom pattern $P$ is mapped to $P'$ by one of the above transformations, then the corresponding atoms are {\em functionally equivalent}, i.e., $\mathcal{A}_P(x) = \mathcal{A}_{P'}(x)$ for all $x \in \R^N$. This creates equivalence classes of atom patterns of degree $k$ and we denote the set of these equivalence classes by $\mathcal{P}$. Since the atoms corresponding to atom patterns within an equivalence class are all equivalent as functions, there is a well-defined action of the symmetric group $\Sigma_N$ on $\mathcal{P}$: For any equivalence class $C \in \mathcal{P}$ and $\sigma \in \Sigma_N$, $\sigma C$ is defined as the equivalence class containing $\sigma P$ for any atom pattern $P \in C$. Since there is no ambiguity in the functions, for every $C\in \mathcal{P}$, we will use the notation $\mathcal{A}_C$ to denote the function represented by the atom $\mathcal{A}_P$ for any atom pattern $P \in C$. We use the standard notation $\mathcal{P}/\Sigma_N$ to denote the set of orbits within $\mathcal{P}$ under the action of $\Sigma_N$.

For any orbit $\mathcal{O} \in \mathcal{P}/\Sigma_N$, we define the function $$G_{\mathcal{O}} = \sum_{C \in \mathcal{O}}\mathcal{A}_C.$$ Since we consider the sum over the entire orbit, we observe that $G_{\mathcal{O}}(x)$ is invariant under coordinate permutations. Just like in Rueß et al~\cite{ruess2026shallower}, we wish to express $\MAX_N$ as a linear combination $\sum_{\mathcal{O}\in \mathcal{P}/\Sigma_N}\alpha_{\mathcal{O}}G_{\mathcal{O}}$. Moreover, after an arbitrary selection of a representative $P_C \in C$ for every equivalence class $C \in \mathcal{P}$, Lemma~\ref{lem:atom-max-linear} implies the following.

\begin{proposition}
    For any collection of real numbers $\alpha_{\mathcal{O}} \in \R$ indexed by the orbits $\mathcal{O}\in \mathcal{P}/\Sigma_N$, there exists an index set $\mathcal{T}$ indexing a set of ordered pairs $\{(\eta^T_L, \eta^T_R): T \in \mathcal{T}\}$ of vectors in $\{y \in \Z^N_+: \|y\|_1 = k\}$ and a set of real numbers $\{c_T\in \R: T \in \mathcal{T}\}$, such that 
    $$\sum_{\mathcal{O}\in \mathcal{P}/\Sigma_N}\alpha_{\mathcal{O}}G_{\mathcal{O}}(x) = \sum_{T \in \mathcal{T}}c_T \max\{\langle \eta^T_L, x\rangle, \langle \eta^T_R, x\rangle\}\qquad \forall x \in \mathcal{C}.$$
\medskip

    Moreover, each $c_T$ is given by a linear combination of the $\alpha_{\mathcal{O}}$ coefficients. Therefore, there exists a matrix $M \in \R^{\mathcal{T}\times (\mathcal{P}/\Sigma_N)}$ such that \begin{equation}\label{eq:main-linear-sys}M\alpha = c,\end{equation} where $\alpha \in \R^\mathcal{P}$ is obtained by collecting all the $\alpha_P, P\in \mathcal{P}$ into a vector, and $c \in \R^\mathcal{T}$ is obtained from $c_T, T \in \mathcal{T}$.
\end{proposition}

Note that the linear system~\eqref{eq:main-linear-sys} is not a trivial system since two or more atom patterns from different orbits may result in the same $\eta_L, \eta_R$ vectors. For example, the atom patterns $[(1,2), (2,3) | (2,3), (3,4)]$ and $[(1,3), (2,4) | (2,4), (3,3)]$ of degree $k=2$ over $N = 4$ are not equivalent under the equivalence relations defined above, but they both give rise to the same vectors $\eta_L = [1,1,0,0]$ and $\eta_R = [0,1,1,0]$. In other words, the symmetries we consider above do not cover all the symmetries for functional equivalence of the atoms.

We, therefore, arrive at the conclusion that $$\MAX_N = \sum_{\mathcal{O}\in \mathcal{P}/\Sigma_N}\alpha_{\mathcal{O}}G_{\mathcal{O}},$$ if and only if $$x_1 = \MAX_N = \sum_{\mathcal{O}\in \mathcal{P}/\Sigma_N}\alpha_{\mathcal{O}}G_{\mathcal{O}} = \sum_{T \in \mathcal{T}}c_T \max\{\langle \eta^T_L, x\rangle, \langle \eta^T_R, x\rangle\},$$ for all $x$ in the sorted chamber $\mathcal{C}$, i.e., \begin{equation}\label{eq:setup}x_1 = \sum_{T \in \mathcal{T}}c_T \max\{\langle \eta^T_L, x\rangle, \langle \eta^T_R, x\rangle\}\qquad \forall x \in \mathcal{C}.\end{equation} The left hand side is a linear function, whereas the right hand side is a piecewise linear function with possible ``breakpoints" or nondifferentiabilities inside $\mathcal{C}$. To force equality, we make the following definition.

\begin{definition}
We define an atom represented by a pattern $P$ to be {\em ambiguous} over $\mathcal{C}$ if there exist $x, x' \in \mathcal C$ such that $L_P(x) > R_P(x)$ and $L_P(x') < R_P(x')$, i.e., neither term inside the maximum dominates the other. An atom that is not ambiguous is said to be {\em unambiguous}.
\end{definition}

Equivalently, the pattern $P$ is unambiguous if and only if $\eta_R - \eta_L \in \mathcal{C}^\circ \cup -\mathcal{C}^\circ$ where $\eta_L, \eta_R$ are the vectors guaranteed by Lemma~\ref{lem:atom-max-linear}, i.e., the difference between $\eta_R$ and $\eta_L$ is in the polar or negative polar of the sorted chamber $\mathcal{C}$. We partition our index set $\mathcal{T} = \mathcal{T}_A \uplus \mathcal{T}_U$ defined by $\mathcal{T}_U := \{T \in \mathcal{T}: \eta^T_R - \eta^T_L \in \mathcal{C}^\circ\cup -\mathcal{C}^\circ\}$ as the indices of the unambiguous terms in $\sum_{T \in \mathcal{T}}c_T \max\{\langle \eta^T_L, x\rangle, \langle \eta^T_R, x\rangle\}$ and $\mathcal{T}_A = \mathcal{T} \setminus \mathcal{T}_U$ as the indices of the ambiguous terms. With this setup, a sufficient condition for~\eqref{eq:setup} to hold is the following:

\begin{equation}\label{eq:sufficient-cond}x_1 = \sum_{T \in \mathcal{T}_U}c_T \max\{\langle \eta^T_L, x\rangle, \langle \eta^T_R, x\rangle\}, \qquad c_T = 0 \quad \forall T \in \mathcal{T}_A.\end{equation}

Note that $\max\{\langle \eta^T_L, x\rangle, \langle \eta^T_R, x\rangle\}$ is a linear function over $\mathcal{C}$ when the term is unambiguous, i.e., $T \in \mathcal{T}_U$. In particular, the equation $x_1 = \sum_{T \in \mathcal{T}_U}c_T \max\{\langle \eta^T_L, x\rangle, \langle \eta^T_R, x\rangle\}$ above is equivalent to a system of linear equations in $c_T$. Using the linear relation~\eqref{eq:main-linear-sys}, the sufficient condition~\eqref{eq:sufficient-cond} gives us an explicit linear system in the coefficients $\alpha_{\mathcal{O}}$. We solve this system in this paper to arrive at our two-hidden-layer representations for $\MAX_5, \MAX_6, \MAX_7$ and $\MAX_8$. The sizes of these linear systems are shown in Table~\ref{tab:system-sizes}. 

\begin{table}[htbp]
\centering
\begin{tabular}{ccrr}
\toprule
$N$ & $k$ & \# constraints & \# variables \\
\midrule
5 & 2 & 20     & 131      \\
6 & 2 & 41     & 144     \\
7 & 3 & 1,057  & 4,469    \\
8 & 4 & 21,953 & 193,623  \\
\bottomrule
\end{tabular}
\caption{Dimensions of the linear systems}
\label{tab:system-sizes}
\end{table}

\subsection{Comparison of this paper to~\cite{ruess2026shallower}} 

As mentioned in the abstract, Rueß et al~\cite{ruess2026shallower} obtain such identities for $\MAX_9$ and $\MAX_{10}$ as well. These two cases remain computationally inaccessible to our methods and resources. In particular, the linear systems are too big to solve using exact rational arithmetic. For $N=9$, the linear systems with $k=4$ have $51,984$ constraints and $210,540$ variables; for $N=10$, the linear systems with $k=4$ have $112,837$ constraints and $216,428$ variables. We were unable to solve the systems with exact rational arithmetic in the amount of compute time available to us. Moreover, for $\MAX_8$, Rueß et al~\cite{ruess2026shallower} are able to obtain a more compact representation in the sense that the atoms are of degree $k=3$ instead of $k=4$ as in our representation. In other words, they provide a representation with sums of only three two-term coordinate maxima, as opposed to sums of four two-term coordinate maxima in our identities.

On a more technical point, in contrast to the linear system from Rueß et al~\cite{ruess2026shallower} described above, our linear system takes into account some additional functional symmetries of the atoms described at the beginning of this subsection. The symmetry arising out of switching the left and right sides of the atom pattern is also used in Rueß et al~\cite{ruess2026shallower} (see Section 4.1.1 in their paper), but the other symmetries seem to be unique to our approach. Moreover, in our linear system, we do not express the two term maxima as a sum of a linear term and a ReLU term, and then force the coefficients of only the ambiguous two term maxima to zero. In contrast, Rueß et al~\cite{ruess2026shallower} first extract a linear term out from the two term maxima and then force {\em all} the coefficients of the hinged functions to zero, as opposed to distinguishing between ambiguous and unambiguous terms. Remark 4.1 in their paper points out that they also take into account ambiguous and unambiguous terms, but the unambiguous terms that are not trivially zero over $\mathcal{C}$ are simply dropped from consideration, whereas in our linear system all the unambiguous terms are allowed to potentially contribute. This is likely the main reason that our representations are different from those of~\cite{ruess2026shallower}. Rueß et al~\cite{ruess2026shallower} have a couple of other simplifications they make for the hinge terms $\max\{0,d\}$ in~\eqref{eq:ruess-equation} which have no analogue for us since we do not extract a linear term from the two term maxima;  see Remark 4.1 in~\cite{ruess2026shallower} for these details.

\section{An exact seven-class identity for $\MAX_5$}

As pointed out in Table~\ref{tab:system-sizes}, we were able to find representations of $\MAX_5$ with $k=2$. After quotienting out with respect to the symmetries described in Section~\ref{sec:our-approach}, and then computing the orbits of these equivalence classes under the action of the symmetric group $\Sigma_5$, we end up with 131 orbits of equivalance classes of atom patterns, i.e., $|\mathcal{P}/\Sigma_5| = 131$ in the notation introduced in Section~\ref{sec:our-approach}. We made certain lexicographically minimal choices for a canonical representative in each orbit to obtain atom patterns $P_1, \ldots, P_{131}$. Below we use the notation and terminology introduced in Section~\ref{sec:our-approach} and $\mathcal{O}[P]$ denotes the orbit of the equivalence class containing the atom pattern $P$.

\begin{theorem}[Exact $\MAX_5$ identity]\label{thm:max5}
We have the following identity.
\begin{align}\label{eq:max5}
2\MAX_5={}&-\frac25 G_{\mathcal O[P_0]}
+\frac1{30} G_{\mathcal O[P_3]}
+\frac1{120} G_{\mathcal O[P_{16}]}
+\frac1{60} G_{\mathcal O[P_{22}]}\notag\\
&+\frac1{60} G_{\mathcal O[P_{37}]}
-\frac1{60} G_{\mathcal O[P_{45}]}
-\frac1{60} G_{\mathcal O[P_{81}]},
\end{align} where
\begin{align*}
 P_0&=[11,11\mid 11,11],\\
 P_3&=[11,11\mid 11,23],\\
 P_{16}&=[11,11\mid 23,45],\\
 P_{22}&=[11,12\mid 11,34],\\
 P_{37}&=[11,12\mid 13,45],\\
 P_{45}&=[11,12\mid 23,45],\\
 P_{81}&=[11,23\mid 12,45].
\end{align*}

\end{theorem}

\begin{remark}[Nonuniqueness of the representation]\label{rem:max5-nonunique}
The identity above is one exact representation supported on the selected seven orbits.  Its coefficients are unique if one restricts to these seven orbits because the corresponding seven columns in the linear system are linearly independent, but the representation is not unique in the full set of $131$ orbits. The following gives a different exact representation supported on only five orbits.
\end{remark}

\begin{theorem}[Exact $\MAX_5$ identity]\label{thm:max5-2}
We have the following identity.

\begin{align}\label{eq:max5-alternative}
2\MAX_5={}&-\frac25 G_{\mathcal O[P_0]}
+\frac1{15} G_{\mathcal O[P_3]}
-\frac1{40} G_{\mathcal O[P_{69}]}+\frac1{20} G_{\mathcal O[P_{95}]}
-\frac1{30} G_{\mathcal O[P_{97}]}.
\end{align}
\begin{align*}
 P_0&=[11,11\mid 11,11],\\
 P_3&=[11,11\mid 11,23],\\
 P_{69}&=[11,22\mid 34,35],\\
 P_{95}&=[11,23\mid 24,25],\\
 P_{97}&=[11,23\mid 25,34].
\end{align*}

\end{theorem}

\section{An exact seven-class identity for $\MAX_6$}

As pointed out in Table~\ref{tab:system-sizes}, we were able to find representations of $\MAX_6$ with $k=2$. After quotienting out with respect to the symmetries described in Section~\ref{sec:our-approach}, and then computing the orbits of these equivalence classes under the action of the symmetric group $\Sigma_6$, we end up with 144 orbits of equivalance classes of atom patterns, i.e., $|\mathcal{P}/\Sigma_6| = 144$ in the notation introduced in Section~\ref{sec:our-approach}. We made certain lexicographically minimal choices for a canonical representative in each orbit to obtain atom patterns $P_1, \ldots, P_{144}$. Below we use the notation and terminology introduced in Section~\ref{sec:our-approach} and $\mathcal{O}[P]$ denotes the orbit of the equivalence class containing the atom pattern $P$.

\begin{theorem}[Exact $\MAX_6$ identity]\label{thm:max6}
We have the following identity.
\begin{align}\label{eq:max6}
2\MAX_6={}&-\frac13 G_{\mathcal O[P_0]}
+\frac1{90} G_{\mathcal O[P_{22}]}
+\frac1{180}G_{\mathcal O[P_{51}]}
-\frac1{360}G_{\mathcal O[P_{71}]}\notag\\
&+\frac1{180}G_{\mathcal O[P_{130}]}
-\frac1{90}G_{\mathcal O[P_{135}]}
+\frac1{720}G_{\mathcal O[P_{136}]}.
\end{align} where
\begin{align*}
 P_0&=[11,11\mid 11,11],\\
 P_{22}&=[11,12\mid 11,34],\\
 P_{51}&=[11,12\mid 34,56],\\
 P_{71}&=[11,22\mid 34,56],\\
 P_{130}&=[12,13\mid 14,56],\\
 P_{135}&=[12,13\mid 24,56],\\
 P_{136}&=[12,13\mid 45,46].
\end{align*}

\end{theorem}

\begin{remark}[Nonuniqueness of the representation]\label{rem:max5-nonunique}
The identity above is one exact representation supported on the selected seven orbits.  Its coefficients are unique if one restricts to these seven orbits because the corresponding seven columns in the linear system are linearly independent, but the representation is not unique in the full set of $144$ orbits. The following gives a different exact representation supported on nine orbits.
\end{remark}

\begin{theorem}[Exact $\MAX_6$ identity]\label{thm:max6-2}
We have the following identity.

\begin{align}\label{eq:max6-alternative}
2\MAX_6={}&\frac1{45} G_{\mathcal O[P_{15}]}
-\frac2{45} G_{\mathcal O[P_{44}]}
+\frac1{90} G_{\mathcal O[P_{47}]}\notag\\
&+\frac2{15} G_{\mathcal O[P_{56}]}
+\frac1{180} G_{\mathcal O[P_{69}]}
+\frac1{720} G_{\mathcal O[P_{118}]}\notag\\
&+\frac1{180} G_{\mathcal O[P_{130}]}
-\frac1{180} G_{\mathcal O[P_{135}]}
-\frac1{360} G_{\mathcal O[P_{142}]},
\end{align} where
$$\begin{array}{lll}
 P_{15}=[11,11\mid 23,24], &P_{44}=[11,12\mid 23,44], &P_{47}=[11,12\mid 33,44],\\
 P_{56}=[11,22\mid 12,12], &P_{69}=[11,22\mid 34,35], &P_{118}=[12,12\mid 34,56],\\
 P_{130}=[12,13\mid 14,56], &P_{135}=[12,13\mid 24,56], &P_{142}=[12,34\mid 13,56],\\
\end{array}
$$
\end{theorem}

\section{An exact 109-class identity for $\MAX_7$}

As pointed out in Table~\ref{tab:system-sizes}, we were able to find representations of $\MAX_7$ with $k=3$. After quotienting out with respect to the symmetries described in Section~\ref{sec:our-approach}, and then computing the orbits of these equivalence classes under the action of the symmetric group $\Sigma_7$, we end up with 4469 orbits of equivalance classes of atom patterns, i.e., $|\mathcal{P}/\Sigma_7| = 4469$ in the notation introduced in Section~\ref{sec:our-approach}. We made certain lexicographically minimal choices for a canonical representative in each orbit to obtain atom patterns $P_1, \ldots, P_{4469}$. Below we use the notation and terminology introduced in Section~\ref{sec:our-approach} and $\mathcal{O}[Q]$ denotes the orbit of the equivalence class containing the atom pattern $Q$.

\begin{theorem}[Exact three-summand representation of $\MAX_7$]
\label{thm:max7}
We have the following identity
\begin{equation}\label{eq:max7identity}
  2\MAX_7=\sum_{j=1}^{109}c_j\, G_{\mathcal{O}[Q_j]},
\end{equation} where the atom patterns $Q_1,\ldots,Q_{109}$ and the coefficients
$c_1,\ldots,c_{109}\in\mathbb Q$ are listed in
\Cref{tab:max7-complete-coefficients} in Appendix~\ref{app:max7-coefficients}. 
\end{theorem}

\section{An exact 1,290-class identity for $\MAX_8$}
\label{sec:max8}

As pointed out in Table~\ref{tab:system-sizes}, we were able to find representations of $\MAX_8$ with $k=4$. After quotienting out with respect to the symmetries described in Section~\ref{sec:our-approach}, and then computing the orbits of these equivalence classes under the action of the symmetric group $\Sigma_8$, we end up with 193623 orbits of equivalance classes of atom patterns, i.e., $|\mathcal{P}/\Sigma_8| = 193263$ in the notation introduced in Section~\ref{sec:our-approach}. We made certain lexicographically minimal choices for a canonical representative in each orbit to obtain atom patterns $P_1, \ldots, P_{193623}$. Below we use the notation and terminology introduced in Section~\ref{sec:our-approach} and $\mathcal{O}[R]$ denotes the orbit of the equivalence class containing the atom pattern $R$.

\begin{theorem}[Exact four-summand representation of $\MAX_8$]
\label{thm:max8}
We have the following identity
\begin{equation}\label{eq:max8identity}
  2\MAX_8=\sum_{j=1}^{1290}d_j\, G_{\mathcal O[R_j]}.
\end{equation} where the atom patterns $R_1,\ldots,R_{1290}$ and the coefficients $d_1,\ldots,d_{1290}\in\mathbb Q$ are listed in \Cref{tab:max8-complete-coefficients} in Appendix~\ref{app:max8-coefficients}.
\end{theorem}

\section{Conclusion}

The central idea used in this paper, as well as in the work of Rueß at al~\cite{ruess2026shallower}, is that a global piecewise-linear identity can
be certified by finite exact linear algebra. Symmetric orbit-sums reduce the problem to one sorted
chamber. The approach of Rueß at al and this paper differ in the way nonlinear terms are handled on this sorted chamber. This paper also uses further functional symmetries beyond coordinate permutations to reduce the size of the linear system. We emphasize that a solution to the corresponding linear system is a sufficient certificate for representation using ReLU networks with two hidden layers; it is not a complete characterization. In other words, nonexistence of a solution does not imply nonexistence of such a representation.  

While our approach provided such representations for $\MAX_5, \MAX_6, \MAX_7$ and $\MAX_8$, Rue{\ss} et al.~\cite{ruess2026shallower} are able to obtain these representations for $\MAX_9$ and $\MAX_{10}$ as well. Our approach was computationally inaccessible for $\MAX_9$ and $\MAX_{10}$ given the number of orbits and equivalence classes in our method. Moreover, while our representation for $\MAX_8$ uses $k=4$, i.e., sum of four two-term coordinate maxima in the first layer, Rue{\ss} et al.~\cite{ruess2026shallower} obtain a representation using sums of only three two-term coordinate maxima for $\MAX_8$.

\bibliographystyle{alpha}
\bibliography{full-bib}

\clearpage

\appendix

\section{The complete $\MAX_7$ coefficient list}\label{app:max7-coefficients}

Table~\ref{tab:max7-complete-coefficients} gives the 109 nonzero coefficients and corresponding atom patterns inTheorem~\ref{thm:max7}.

\input{max7_coefficients.tex}

\section{The complete $\MAX_8$ coefficient list}\label{app:max8-coefficients}

Table~\ref{tab:max8-complete-coefficients} gives the 1,290 nonzero coefficients and corresponding atom patterns in
Theorem~\ref{thm:max8}.  

\input{max8_coefficients.tex}

\end{document}

%% file: max7_coefficients.tex
\begingroup
\captionof{table}{Complete coefficient list for the exact $\MAX_7$ identity. Each entry is $j,\ c_j,\ Q_j = \texttt{[left $\mid$ right]}$. To denote the entire orbit of equivalence classes under the symmetries we consider, we use abstract letters to denote the patterns. The coefficient corresponds to the coefficient of the function obtained by summing all atoms obtained when the letters range over all possible values, i.e., ${\tt a, b, c}$ etc. range over $\{1, \ldots, 7\}$.}\label{tab:max7-complete-coefficients}
\vspace{0.4\baselineskip}
\setlength{\columnsep}{1.2em}
\setlength{\columnseprule}{0pt}
\begin{multicols}{2}
\raggedcolumns
\fontsize{8.0}{8.6}\selectfont
\setlength{\parindent}{0pt}
\setlength{\parskip}{0pt}
\newcommand{\coeffentry}[4]{%
  \noindent\makebox[2.25em][r]{#1.}\hspace{0.35em}%
  \makebox[8.6em][r]{\(\scriptstyle #2\)}\hspace{0.55em}%
  \mbox{\texttt{#3\,|\,#4}}\par%
}
\coeffentry{1}{\tfrac{13}{42}}{aa aa aa}{aa aa aa}
\coeffentry{2}{\tfrac{31}{840}}{aa aa aa}{aa aa bc}
\coeffentry{3}{\tfrac{29}{20160}}{aa aa aa}{bc bd cd}
\coeffentry{4}{-\tfrac{73}{1680}}{aa aa bc}{aa aa bc}
\coeffentry{5}{\tfrac{29}{13440}}{aa aa bc}{bc cd dd}
\coeffentry{6}{-\tfrac{1}{1344}}{aa aa bc}{bc de df}
\coeffentry{7}{-\tfrac{29}{6720}}{aa aa bc}{bd cd dd}
\coeffentry{8}{\tfrac{29}{13440}}{aa aa bd}{bc cc dd}
\coeffentry{9}{-\tfrac{29}{13440}}{aa aa cd}{bc bc dd}
\coeffentry{10}{\tfrac{1}{1680}}{aa bb cc}{bc de df}
\coeffentry{11}{-\tfrac{1}{448}}{aa bb cc}{cc de df}
\coeffentry{12}{-\tfrac{23}{40320}}{aa bb cc}{de df dg}
\coeffentry{13}{\tfrac{1}{5376}}{aa bb cc}{de df ef}
\coeffentry{14}{\tfrac{53}{30240}}{aa bb cc}{de df fg}
\coeffentry{15}{-\tfrac{1}{3360}}{aa bb cc}{de ef ff}
\coeffentry{16}{-\tfrac{1}{3840}}{aa bb cc}{df ef ef}
\coeffentry{17}{\tfrac{1}{1792}}{aa bb cc}{df ef ff}
\coeffentry{18}{-\tfrac{1}{1680}}{aa bb cd}{ab ce cf}
\coeffentry{19}{\tfrac{1}{1680}}{aa bb cd}{ab cf de}
\coeffentry{20}{\tfrac{1}{224}}{aa bb cd}{bb ce cf}
\coeffentry{21}{-\tfrac{1}{336}}{aa bb cd}{bb cf de}
\coeffentry{22}{\tfrac{13}{30240}}{aa bb cd}{cd ef eg}
\coeffentry{23}{\tfrac{1}{2016}}{aa bb cd}{ce cf cg}
\coeffentry{24}{-\tfrac{13}{4032}}{aa bb cd}{ce cf fg}
\coeffentry{25}{\tfrac{7}{5760}}{aa bb cd}{ce ef eg}
\coeffentry{26}{\tfrac{13}{60480}}{aa bb cd}{dd ef eg}
\coeffentry{27}{\tfrac{13}{60480}}{aa bb dd}{cd ef eg}
\coeffentry{28}{-\tfrac{41}{20160}}{aa bb de}{cd cf fg}
\coeffentry{29}{\tfrac{17}{6720}}{aa bb de}{ce df ff}
\coeffentry{30}{-\tfrac{1}{1344}}{aa bb de}{ce ef ff}
\coeffentry{31}{-\tfrac{1}{480}}{aa bb de}{cf ef ff}
\coeffentry{32}{\tfrac{1}{3360}}{aa bb ef}{cd cd df}
\coeffentry{33}{-\tfrac{1}{3360}}{aa bb ef}{cd cd ef}
\coeffentry{34}{\tfrac{11}{13440}}{aa bb ef}{cd cd ff}
\coeffentry{35}{-\tfrac{1}{6720}}{aa bb ef}{cd cf de}
\coeffentry{36}{-\tfrac{1}{6720}}{aa bb ef}{cd cf df}
\coeffentry{37}{\tfrac{1}{2240}}{aa bb ef}{cd de df}
\coeffentry{38}{-\tfrac{3}{2240}}{aa bb ef}{cd de ff}
\coeffentry{39}{\tfrac{3}{2240}}{aa bb ef}{cd df df}
\coeffentry{40}{-\tfrac{1}{6720}}{aa bb ef}{cd df ef}
\coeffentry{41}{\tfrac{1}{960}}{aa bb ef}{ce de ff}
\coeffentry{42}{-\tfrac{1}{960}}{aa bb ef}{ce df df}
\coeffentry{43}{-\tfrac{1}{2240}}{aa bb ef}{cf df df}
\coeffentry{44}{\tfrac{97}{120960}}{aa bc bc}{de df dg}
\coeffentry{45}{-\tfrac{1}{1120}}{aa bc bc}{de df fg}
\coeffentry{46}{\tfrac{17}{60480}}{aa bc bd}{be bf bg}
\coeffentry{47}{\tfrac{13}{20160}}{aa bc bd}{be bf fg}
\coeffentry{48}{-\tfrac{31}{10080}}{aa bc bd}{be ef eg}
\coeffentry{49}{-\tfrac{13}{20160}}{aa bc bd}{bf bg de}
\coeffentry{50}{\tfrac{7}{2880}}{aa bc bd}{bg de df}
\coeffentry{51}{-\tfrac{23}{7560}}{aa bc bd}{de df dg}
\coeffentry{52}{-\tfrac{13}{20160}}{aa bc be}{cd ef eg}
\coeffentry{53}{\tfrac{1}{6720}}{aa bc cd}{be bf fg}
\coeffentry{54}{\tfrac{29}{20160}}{aa bc cd}{be ef eg}
\coeffentry{55}{\tfrac{11}{6720}}{aa bc cd}{be ef fg}
\coeffentry{56}{\tfrac{13}{20160}}{aa bc cd}{bf ce fg}
\coeffentry{57}{\tfrac{7}{2880}}{aa bc de}{bd bf fg}
\coeffentry{58}{\tfrac{1}{672}}{aa bc de}{bd df dg}
\coeffentry{59}{-\tfrac{7}{2880}}{aa bc de}{bf cd fg}
\coeffentry{60}{\tfrac{1}{3360}}{aa bc de}{cc ef ff}
\coeffentry{61}{\tfrac{13}{10080}}{aa bc ef}{bd be eg}
\coeffentry{62}{\tfrac{1}{3360}}{aa bc ef}{cc de df}
\coeffentry{63}{-\tfrac{1}{840}}{aa bc ef}{cc de ff}
\coeffentry{64}{-\tfrac{1}{1680}}{aa bc ef}{cc df df}
\coeffentry{65}{-\tfrac{1}{560}}{aa bc fg}{be bf cd}
\coeffentry{66}{\tfrac{1}{560}}{aa bc fg}{be cd ef}
\coeffentry{67}{-\tfrac{1}{840}}{aa be cd}{cf de ff}
\coeffentry{68}{\tfrac{1}{1120}}{aa be cd}{df ef ff}
\coeffentry{69}{\tfrac{3}{1120}}{aa cd ef}{be df ff}
\coeffentry{70}{-\tfrac{1}{672}}{aa cd ef}{bf de ef}
\coeffentry{71}{\tfrac{1}{672}}{aa ce df}{be bf df}
\coeffentry{72}{\tfrac{1}{3360}}{aa ce df}{be de ff}
\coeffentry{73}{\tfrac{1}{560}}{aa ce df}{bf bf de}
\coeffentry{74}{\tfrac{1}{1680}}{aa ce df}{bf de ef}
\coeffentry{75}{-\tfrac{1}{420}}{aa ce df}{bf de ff}
\coeffentry{76}{\tfrac{1}{1120}}{aa ce df}{bf df ef}
\coeffentry{77}{-\tfrac{1}{1120}}{aa cf de}{bd be bf}
\coeffentry{78}{-\tfrac{1}{840}}{aa cf de}{bd ef ef}
\coeffentry{79}{-\tfrac{1}{1680}}{aa cf de}{be bf cd}
\coeffentry{80}{\tfrac{1}{840}}{aa cf de}{be cd ff}
\coeffentry{81}{\tfrac{1}{1680}}{aa cf de}{be cf df}
\coeffentry{82}{-\tfrac{1}{1120}}{aa cf de}{bf df ef}
\coeffentry{83}{\tfrac{1}{2520}}{aa de fg}{bc bd bf}
\coeffentry{84}{-\tfrac{13}{60480}}{aa ef eg}{bc bd cd}
\coeffentry{85}{\tfrac{13}{20160}}{aa ef eg}{bc cd dd}
\coeffentry{86}{-\tfrac{13}{20160}}{aa ef fg}{bc bd gg}
\coeffentry{87}{-\tfrac{13}{20160}}{aa ef gg}{bc bd fg}
\coeffentry{88}{\tfrac{13}{20160}}{aa eg fg}{bc bd ef}
\coeffentry{89}{-\tfrac{47}{60480}}{ab ab cd}{ce cf cg}
\coeffentry{90}{\tfrac{23}{20160}}{ab ab cd}{ce cf fg}
\coeffentry{91}{-\tfrac{1}{40320}}{ab ab cd}{ce ef eg}
\coeffentry{92}{-\tfrac{1}{4032}}{ab ab de}{cd cf fg}
\coeffentry{93}{\tfrac{47}{60480}}{ab ac ad}{ae ef eg}
\coeffentry{94}{-\tfrac{1}{2520}}{ab ac ae}{cd ef eg}
\coeffentry{95}{\tfrac{1}{2520}}{ab ac ae}{cd ef fg}
\coeffentry{96}{\tfrac{1}{6048}}{ab ac bc}{de df fg}
\coeffentry{97}{\tfrac{1}{2520}}{ab ac cd}{ae af fg}
\coeffentry{98}{-\tfrac{1}{2880}}{ab ac cd}{ae ef eg}
\coeffentry{99}{-\tfrac{1}{2016}}{ab ac cd}{ae ef fg}
\coeffentry{100}{-\tfrac{1}{5040}}{ab ac de}{ad af fg}
\coeffentry{101}{-\tfrac{1}{1680}}{ab ac de}{af cd fg}
\coeffentry{102}{-\tfrac{1}{1440}}{ab ac ef}{ae cd eg}
\coeffentry{103}{-\tfrac{13}{20160}}{ab ac ef}{df eg gg}
\coeffentry{104}{\tfrac{1}{2016}}{ab ac fg}{ae cd ef}
\coeffentry{105}{\tfrac{1}{1008}}{ab ac fg}{af cd de}
\coeffentry{106}{\tfrac{1}{6720}}{ab ac fg}{de dg ef}
\coeffentry{107}{-\tfrac{1}{6720}}{ab ac fg}{de dg eg}
\coeffentry{108}{-\tfrac{1}{840}}{ab ae bc}{bd ef fg}
\coeffentry{109}{\tfrac{1}{5040}}{ab ae cd}{bc ef fg}
\end{multicols}
\endgroup

%% file: max8_coefficients.tex
\begingroup
\captionof{table}{Complete coefficient list for the exact four-summand $\MAX_8$ identity. Each entry is $j,\ d_j,\ R_j = \texttt{[left $\mid$ right]}$. To denote the entire orbit of equivalence classes under the symmetries we consider, we use abstract letters to denote the patterns. The coefficient corresponds to the coefficient of the function obtained by summing all atoms obtained when the letters range over all possible values, i.e., ${\tt a, b, c}$ etc. range over $\{1, \ldots, 8\}$.}\label{tab:max8-complete-coefficients}
\vspace{0.4\baselineskip}
\setlength{\columnsep}{1.2em}
\setlength{\columnseprule}{0pt}
\begin{multicols}{2}
\raggedcolumns
\fontsize{8.0}{8.6}\selectfont
\setlength{\parindent}{0pt}
\setlength{\parskip}{0pt}
\newcommand{\coeffentry}[4]{%
  \noindent\makebox[2.25em][r]{#1.}\hspace{0.35em}%
  \makebox[8.6em][r]{\(\scriptstyle #2\)}\hspace{0.55em}%
  \mbox{\texttt{#3\,|\,#4}}\par%
}
\coeffentry{1}{\tfrac{77741}{1280}}{aa aa aa aa}{aa aa aa aa}
\coeffentry{2}{-\tfrac{8843}{12096}}{aa aa aa aa}{aa aa aa bc}
\coeffentry{3}{\tfrac{28849}{241920}}{aa aa aa aa}{bc bd cd cd}
\coeffentry{4}{-\tfrac{28849}{483840}}{aa aa aa aa}{bd bd cd cd}
\coeffentry{5}{-\tfrac{435091}{362880}}{aa aa aa bc}{aa aa aa bc}
\coeffentry{6}{\tfrac{28849}{362880}}{aa aa aa bc}{bc bd cd dd}
\coeffentry{7}{\tfrac{28849}{362880}}{aa aa aa bc}{bc cd cd dd}
\coeffentry{8}{\tfrac{28849}{362880}}{aa aa aa bc}{bd cc dd dd}
\coeffentry{9}{-\tfrac{28849}{181440}}{aa aa aa bc}{bd cd cd dd}
\coeffentry{10}{-\tfrac{28849}{362880}}{aa aa aa bc}{bd cd dd dd}
\coeffentry{11}{-\tfrac{28849}{362880}}{aa aa aa cd}{bc bc bc dd}
\coeffentry{12}{-\tfrac{187}{604800}}{aa aa bb bb}{cd ce fg fh}
\coeffentry{13}{\tfrac{971}{201600}}{aa aa bc bc}{aa aa de df}
\coeffentry{14}{\tfrac{1}{20160}}{aa aa bc bc}{de df dg dh}
\coeffentry{15}{\tfrac{101}{604800}}{aa aa bc bc}{de df dg gh}
\coeffentry{16}{-\tfrac{403}{2419200}}{aa aa bc bc}{de dg ef gh}
\coeffentry{17}{\tfrac{971}{80640}}{aa aa bc bd}{aa aa be bf}
\coeffentry{18}{-\tfrac{971}{50400}}{aa aa bc bd}{aa aa bf de}
\coeffentry{19}{-\tfrac{1}{20160}}{aa aa bc bd}{be bf bg bh}
\coeffentry{20}{-\tfrac{7}{345600}}{aa aa bc bd}{be bf fg fh}
\coeffentry{21}{\tfrac{451}{2419200}}{aa aa bc bd}{be bg ef gh}
\coeffentry{22}{-\tfrac{451}{1209600}}{aa aa bc bd}{bf bg de gh}
\coeffentry{23}{\tfrac{451}{2419200}}{aa aa bc bd}{bf de fg gh}
\coeffentry{24}{\tfrac{167}{483840}}{aa aa bc bd}{bg bh de df}
\coeffentry{25}{\tfrac{451}{2419200}}{aa aa bc bd}{bg de df gh}
\coeffentry{26}{-\tfrac{451}{1209600}}{aa aa bc bd}{bg de ef gh}
\coeffentry{27}{-\tfrac{1}{23040}}{aa aa bc bd}{bh de df dg}
\coeffentry{28}{\tfrac{451}{2419200}}{aa aa bc be}{bg bh cd ef}
\coeffentry{29}{-\tfrac{451}{2419200}}{aa aa bc be}{bh cd ef eg}
\coeffentry{30}{\tfrac{451}{2419200}}{aa aa bc be}{bh cd ef fg}
\coeffentry{31}{\tfrac{451}{2419200}}{aa aa bc be}{cd ef eg gh}
\coeffentry{32}{-\tfrac{1}{12600}}{aa aa bc bf}{cd ce fg fh}
\coeffentry{33}{-\tfrac{451}{2419200}}{aa aa bc bf}{cd ce fg gh}
\coeffentry{34}{\tfrac{451}{2419200}}{aa aa bc bf}{cd de fg gh}
\coeffentry{35}{-\tfrac{1}{10080}}{aa aa bc cc}{de df dg dh}
\coeffentry{36}{-\tfrac{101}{302400}}{aa aa bc cc}{de df dg gh}
\coeffentry{37}{\tfrac{403}{1209600}}{aa aa bc cc}{de dg ef gh}
\coeffentry{38}{\tfrac{7}{345600}}{aa aa bc cd}{bf ce fg fh}
\coeffentry{39}{-\tfrac{451}{2419200}}{aa aa bc cd}{bf de fg gh}
\coeffentry{40}{\tfrac{19}{21600}}{aa aa bc dd}{cd ef eg eh}
\coeffentry{41}{\tfrac{1}{16800}}{aa aa bc dd}{cd ef eg gh}
\coeffentry{42}{-\tfrac{8917}{1814400}}{aa aa bc de}{aa aa df dg}
\coeffentry{43}{\tfrac{8917}{1814400}}{aa aa bc de}{aa aa df fg}
\coeffentry{44}{\tfrac{8917}{3628800}}{aa aa bc de}{aa aa dg ef}
\coeffentry{45}{\tfrac{4859}{1036800}}{aa aa bc de}{bc de fg fh}
\coeffentry{46}{-\tfrac{19}{21600}}{aa aa bc de}{bc df dg dh}
\coeffentry{47}{-\tfrac{11}{80640}}{aa aa bc de}{bc ee fg fh}
\coeffentry{48}{-\tfrac{11}{161280}}{aa aa bc de}{cc ee fg fh}
\coeffentry{49}{-\tfrac{527}{403200}}{aa aa bc ee}{bc de fg fh}
\coeffentry{50}{\tfrac{23}{34560}}{aa aa bc ee}{cc de fg fh}
\coeffentry{51}{-\tfrac{451}{2419200}}{aa aa bc ef}{bd be eg eh}
\coeffentry{52}{-\tfrac{451}{2419200}}{aa aa bc ef}{be cd eh fg}
\coeffentry{53}{\tfrac{451}{2419200}}{aa aa bc fg}{be cd ef fh}
\coeffentry{54}{-\tfrac{19}{21600}}{aa aa bd cd}{bc ef eg eh}
\coeffentry{55}{\tfrac{19}{21600}}{aa aa bd cd}{cd ef eg eh}
\coeffentry{56}{\tfrac{797}{1209600}}{aa aa cc de}{bc df dg dh}
\coeffentry{57}{-\tfrac{1291}{1209600}}{aa aa cc ee}{bc de fg fh}
\coeffentry{58}{\tfrac{11}{16128}}{aa aa cd cd}{bd ef eg eh}
\coeffentry{59}{\tfrac{239}{1209600}}{aa aa cd dd}{bc ef eg eh}
\coeffentry{60}{-\tfrac{19}{151200}}{aa aa cd dd}{bc ef eg gh}
\coeffentry{61}{-\tfrac{1861}{1209600}}{aa aa cd dd}{bd ef eg eh}
\coeffentry{62}{-\tfrac{1}{1890}}{aa aa cd ef}{bc dd eg gh}
\coeffentry{63}{\tfrac{1}{1890}}{aa aa cd ef}{bd cd eg gh}
\coeffentry{64}{\tfrac{1}{25200}}{aa aa cd gh}{bc be ef eg}
\coeffentry{65}{-\tfrac{1}{1890}}{aa aa dd ef}{bd cd eg gh}
\coeffentry{66}{\tfrac{991}{806400}}{aa aa de de}{bc bc fg fh}
\coeffentry{67}{\tfrac{11}{80640}}{aa aa de de}{bc cc fg fh}
\coeffentry{68}{-\tfrac{1937}{1209600}}{aa aa de ee}{bc bc fg fh}
\coeffentry{69}{-\tfrac{8917}{7257600}}{aa aa de fg}{aa aa bc bc}
\coeffentry{70}{-\tfrac{1}{50400}}{aa aa de gh}{bc bf cd fg}
\coeffentry{71}{-\tfrac{11}{40320}}{aa aa eg fg}{bc bd fh hh}
\coeffentry{72}{\tfrac{11}{40320}}{aa aa eg fg}{bc bd gh hh}
\coeffentry{73}{\tfrac{23}{201600}}{aa aa eg fh}{bc bd fg hh}
\coeffentry{74}{\tfrac{29}{67200}}{aa aa eg fh}{bc bd fh gh}
\coeffentry{75}{\tfrac{281}{151200}}{aa aa eh fg}{bc bd eg fh}
\coeffentry{76}{-\tfrac{1913}{604800}}{aa aa eh fg}{bc bd fh gh}
\coeffentry{77}{\tfrac{13}{12600}}{aa aa eh fg}{bc bd gh gh}
\coeffentry{78}{-\tfrac{11}{40320}}{aa aa fg gh}{bc bd eg hh}
\coeffentry{79}{-\tfrac{11}{4800}}{aa aa fg hh}{bc bd eg fh}
\coeffentry{80}{\tfrac{1937}{604800}}{aa aa fg hh}{bc bd eg gh}
\coeffentry{81}{-\tfrac{1147}{604800}}{aa aa fg hh}{bc bd eh gh}
\coeffentry{82}{-\tfrac{211}{302400}}{aa aa fh gh}{bc bd ef eg}
\coeffentry{83}{-\tfrac{37}{43200}}{aa aa fh gh}{bc bd eg eg}
\coeffentry{84}{\tfrac{431}{302400}}{aa aa fh gh}{bc bd eg eh}
\coeffentry{85}{\tfrac{2069}{604800}}{aa aa fh gh}{bc bd eg fg}
\coeffentry{86}{-\tfrac{431}{302400}}{aa aa fh gh}{bc bd eg fh}
\coeffentry{87}{\tfrac{13}{100800}}{aa aa fh gh}{bc bd eh eh}
\coeffentry{88}{-\tfrac{7}{3456}}{aa aa gh hh}{bc bd ef fg}
\coeffentry{89}{-\tfrac{13}{151200}}{aa ab bb cd}{cd ef eg eh}
\coeffentry{90}{\tfrac{19}{241920}}{aa ab bb cd}{cd ef eg gh}
\coeffentry{91}{\tfrac{1}{7200}}{aa ab bb cd}{ce cf cg ch}
\coeffentry{92}{\tfrac{1}{21600}}{aa ab bb cd}{dd ef eg eh}
\coeffentry{93}{-\tfrac{1}{25200}}{aa ab bb cd}{dd ef eg gh}
\coeffentry{94}{\tfrac{1}{33600}}{aa ab bb dd}{cd ef eg eh}
\coeffentry{95}{\tfrac{1}{22400}}{aa ab bb dd}{cd ef eg gh}
\coeffentry{96}{\tfrac{1}{6300}}{aa ab bb fg}{cd ce gh hh}
\coeffentry{97}{-\tfrac{1}{6300}}{aa ab bb gh}{cd ce fg fh}
\coeffentry{98}{\tfrac{1}{3150}}{aa ab bb gh}{cd ce fg hh}
\coeffentry{99}{-\tfrac{1}{6300}}{aa ab bb gh}{cd ce fh hh}
\coeffentry{100}{\tfrac{1}{6300}}{aa ab bb hh}{cd ce fh gh}
\coeffentry{101}{\tfrac{1}{10080}}{aa bb bb cd}{ce cf cg ch}
\coeffentry{102}{\tfrac{5}{12096}}{aa bb bb cd}{ce cf fg fh}
\coeffentry{103}{-\tfrac{451}{1209600}}{aa bb bb cd}{ce cg ef gh}
\coeffentry{104}{\tfrac{139}{604800}}{aa bb bb cd}{cf cg ch de}
\coeffentry{105}{\tfrac{451}{1209600}}{aa bb bb cd}{cf cg de gh}
\coeffentry{106}{-\tfrac{5}{24192}}{aa bb bb cd}{cf de fg fh}
\coeffentry{107}{-\tfrac{451}{1209600}}{aa bb bb cd}{cf de fg gh}
\coeffentry{108}{-\tfrac{643}{1209600}}{aa bb bb cd}{cg ch de df}
\coeffentry{109}{\tfrac{451}{1209600}}{aa bb bb cd}{cg de ef gh}
\coeffentry{110}{-\tfrac{1}{25200}}{aa bb bb de}{cd cf fg fh}
\coeffentry{111}{\tfrac{1}{25200}}{aa bb bb de}{cd cf fg gh}
\coeffentry{112}{-\tfrac{1}{25200}}{aa bb bb ef}{cd ce cg gh}
\coeffentry{113}{-\tfrac{1}{7200}}{aa bb bc cc}{de df dg dh}
\coeffentry{114}{-\tfrac{19}{67200}}{aa bb cc dd}{cd ef eg eh}
\coeffentry{115}{-\tfrac{353}{1209600}}{aa bb cc dd}{cd ef eg gh}
\coeffentry{116}{\tfrac{109363}{29030400}}{aa bb cc dd}{ef ef gh gh}
\coeffentry{117}{-\tfrac{1449409}{58060800}}{aa bb cc dd}{ef ef gh hh}
\coeffentry{118}{\tfrac{575587}{19353600}}{aa bb cc dd}{ef ff gh hh}
\coeffentry{119}{-\tfrac{46261}{23224320}}{aa bb cc dd}{eg eh fg fh}
\coeffentry{120}{\tfrac{227537}{29030400}}{aa bb cc dd}{eg fg fh hh}
\coeffentry{121}{\tfrac{7381}{7257600}}{aa bb cc dd}{eg fg gh hh}
\coeffentry{122}{\tfrac{8959}{7257600}}{aa bb cc dd}{eg fh fh gh}
\coeffentry{123}{-\tfrac{8917}{7257600}}{aa bb cc dd}{eh fg fh gh}
\coeffentry{124}{\tfrac{23}{134400}}{aa bb cc dd}{eh fh gh gh}
\coeffentry{125}{\tfrac{3733}{518400}}{aa bb cc de}{bc de fg fh}
\coeffentry{126}{\tfrac{19}{33600}}{aa bb cc de}{bc df dg dh}
\coeffentry{127}{\tfrac{353}{604800}}{aa bb cc de}{bc df dg gh}
\coeffentry{128}{-\tfrac{71}{134400}}{aa bb cc de}{bc dg dh ef}
\coeffentry{129}{\tfrac{83}{241920}}{aa bb cc de}{bc dg ef gh}
\coeffentry{130}{\tfrac{1121167}{58060800}}{aa bb cc de}{de fg fh gh}
\coeffentry{131}{-\tfrac{21451}{3628800}}{aa bb cc de}{de fg gh hh}
\coeffentry{132}{-\tfrac{73159}{1658880}}{aa bb cc de}{de fh gh gh}
\coeffentry{133}{\tfrac{2083229}{58060800}}{aa bb cc de}{de fh gh hh}
\coeffentry{134}{\tfrac{20689}{7257600}}{aa bb cc de}{ee fg fh gh}
\coeffentry{135}{-\tfrac{4951}{403200}}{aa bb cc de}{ee fg gh hh}
\coeffentry{136}{-\tfrac{1091}{453600}}{aa bb cc de}{ee fh gh gh}
\coeffentry{137}{-\tfrac{3251}{302400}}{aa bb cc ee}{bc de fg fh}
\coeffentry{138}{\tfrac{38423}{1036800}}{aa bb cc fg}{de de gh hh}
\coeffentry{139}{\tfrac{1}{120960}}{aa bb cc fg}{de dh eg hh}
\coeffentry{140}{-\tfrac{297947}{7257600}}{aa bb cc fg}{de ee gh hh}
\coeffentry{141}{\tfrac{181}{907200}}{aa bb cc fg}{de eg eh hh}
\coeffentry{142}{\tfrac{41819}{1814400}}{aa bb cc fg}{de eg fh hh}
\coeffentry{143}{-\tfrac{16637}{806400}}{aa bb cc fg}{de eg gh hh}
\coeffentry{144}{-\tfrac{116297}{3628800}}{aa bb cc fg}{de eh gh hh}
\coeffentry{145}{-\tfrac{131}{172800}}{aa bb cc fg}{dg ef eh hh}
\coeffentry{146}{-\tfrac{151}{604800}}{aa bb cc fg}{dg ef gh hh}
\coeffentry{147}{-\tfrac{235}{48384}}{aa bb cc fg}{dg eg eh hh}
\coeffentry{148}{-\tfrac{101}{60480}}{aa bb cc fg}{dg eg gh hh}
\coeffentry{149}{\tfrac{12913}{1036800}}{aa bb cc fg}{dh eg eh hh}
\coeffentry{150}{\tfrac{11651}{1451520}}{aa bb cc fg}{dh eg fh hh}
\coeffentry{151}{\tfrac{1531}{151200}}{aa bb cc fg}{dh eg gh hh}
\coeffentry{152}{\tfrac{24971}{1451520}}{aa bb cc fg}{dh eh gh hh}
\coeffentry{153}{-\tfrac{1243}{145152}}{aa bb cc gh}{de de fg fh}
\coeffentry{154}{-\tfrac{37}{80640}}{aa bb cc gh}{de de fg hh}
\coeffentry{155}{-\tfrac{3919}{1209600}}{aa bb cc gh}{de de fh fh}
\coeffentry{156}{\tfrac{131489}{7257600}}{aa bb cc gh}{de df ef fh}
\coeffentry{157}{\tfrac{30983}{1814400}}{aa bb cc gh}{de ee fg fh}
\coeffentry{158}{-\tfrac{359}{25920}}{aa bb cc gh}{de ee fg hh}
\coeffentry{159}{\tfrac{11}{3150}}{aa bb cc gh}{de ee fh fh}
\coeffentry{160}{\tfrac{7921}{3628800}}{aa bb cc gh}{de ef fg fh}
\coeffentry{161}{\tfrac{9407}{806400}}{aa bb cc gh}{de ef fh fh}
\coeffentry{162}{-\tfrac{147251}{7257600}}{aa bb cc gh}{df dh ef ef}
\coeffentry{163}{-\tfrac{229}{120960}}{aa bb cc gh}{df dh ef eg}
\coeffentry{164}{\tfrac{4999}{1209600}}{aa bb cc gh}{df dh ef eh}
\coeffentry{165}{-\tfrac{193}{113400}}{aa bb cc gh}{df ef eg hh}
\coeffentry{166}{-\tfrac{4979}{483840}}{aa bb cc gh}{df ef eh fg}
\coeffentry{167}{\tfrac{47}{11200}}{aa bb cc gh}{df ef eh gh}
\coeffentry{168}{\tfrac{121}{302400}}{aa bb cc gh}{df ef fg fh}
\coeffentry{169}{\tfrac{149}{103680}}{aa bb cc gh}{df ef fg hh}
\coeffentry{170}{-\tfrac{128911}{7257600}}{aa bb cc gh}{df ef fh fh}
\coeffentry{171}{\tfrac{59641}{7257600}}{aa bb cc gh}{df ef fh gh}
\coeffentry{172}{-\tfrac{401}{172800}}{aa bb cc gh}{df eg eh fh}
\coeffentry{173}{-\tfrac{7459}{806400}}{aa bb cc gh}{df eh eh fg}
\coeffentry{174}{-\tfrac{439}{302400}}{aa bb cc gh}{df eh eh fh}
\coeffentry{175}{\tfrac{11}{6720}}{aa bb cc gh}{dg ef eh fh}
\coeffentry{176}{\tfrac{3443}{345600}}{aa bb cc gh}{dg ef fg hh}
\coeffentry{177}{\tfrac{613}{1209600}}{aa bb cc gh}{dg eg fg hh}
\coeffentry{178}{\tfrac{1697}{1209600}}{aa bb cc gh}{dh ef eh fh}
\coeffentry{179}{-\tfrac{19}{151200}}{aa bb cc gh}{dh eh fh fh}
\coeffentry{180}{-\tfrac{1}{9450}}{aa bb cd cd}{cd ef eg gh}
\coeffentry{181}{-\tfrac{63307}{38707200}}{aa bb cd cd}{eg eh fg fh}
\coeffentry{182}{\tfrac{2741}{4838400}}{aa bb cd cd}{eg fg fh hh}
\coeffentry{183}{-\tfrac{43}{604800}}{aa bb cd cd}{eg fg gh hh}
\coeffentry{184}{\tfrac{58057}{9676800}}{aa bb cd cd}{eg fh fh gh}
\coeffentry{185}{\tfrac{11}{241920}}{aa bb cd cd}{eh fg fh gh}
\coeffentry{186}{\tfrac{59083}{19353600}}{aa bb cd cd}{eh fg gh gh}
\coeffentry{187}{-\tfrac{5861}{483840}}{aa bb cd cd}{eh fg gh hh}
\coeffentry{188}{-\tfrac{1}{241920}}{aa bb cd cd}{eh fh gh gh}
\coeffentry{189}{-\tfrac{19}{67200}}{aa bb cd ce}{ab cf cg ch}
\coeffentry{190}{\tfrac{181}{403200}}{aa bb cd ce}{ab cf cg gh}
\coeffentry{191}{\tfrac{1}{5400}}{aa bb cd ce}{ab cg ch ef}
\coeffentry{192}{\tfrac{83}{241920}}{aa bb cd ce}{ab ch ef eg}
\coeffentry{193}{-\tfrac{1}{12600}}{aa bb cd cf}{ab de fg fh}
\coeffentry{194}{-\tfrac{83}{241920}}{aa bb cd de}{ab cg df gh}
\coeffentry{195}{-\tfrac{1307}{181440}}{aa bb cd ef}{ab cd eg eh}
\coeffentry{196}{\tfrac{11}{134400}}{aa bb cd ef}{ab cd eg gh}
\coeffentry{197}{\tfrac{10093}{3628800}}{aa bb cd ef}{ab cd eh fg}
\coeffentry{198}{-\tfrac{1}{1350}}{aa bb cd fg}{ab ce cf fh}
\coeffentry{199}{-\tfrac{307}{11520}}{aa bb cd fg}{cd eg eh hh}
\coeffentry{200}{\tfrac{26821}{645120}}{aa bb cd fg}{cd eg fh hh}
\coeffentry{201}{-\tfrac{402307}{9676800}}{aa bb cd fg}{cd eg gh hh}
\coeffentry{202}{\tfrac{253357}{9676800}}{aa bb cd fg}{cd eh gh hh}
\coeffentry{203}{\tfrac{1033}{806400}}{aa bb cd fg}{dd eg eh hh}
\coeffentry{204}{-\tfrac{14663}{2419200}}{aa bb cd fg}{dd eg fh hh}
\coeffentry{205}{-\tfrac{7709}{241920}}{aa bb cd fg}{dd eg gh hh}
\coeffentry{206}{-\tfrac{1181}{86400}}{aa bb cd fg}{dd eh gh hh}
\coeffentry{207}{-\tfrac{3941}{115200}}{aa bb cd fg}{de ee gh hh}
\coeffentry{208}{\tfrac{51173}{3870720}}{aa bb cd gh}{cd ef ef fh}
\coeffentry{209}{\tfrac{36919}{2764800}}{aa bb cd gh}{cd ef eh fg}
\coeffentry{210}{\tfrac{58547}{9676800}}{aa bb cd gh}{cd ef eh fh}
\coeffentry{211}{-\tfrac{13}{2419200}}{aa bb cd gh}{cd ef fg fh}
\coeffentry{212}{-\tfrac{1165877}{29030400}}{aa bb cd gh}{cd ef fg hh}
\coeffentry{213}{-\tfrac{1232131}{19353600}}{aa bb cd gh}{cd ef fh fh}
\coeffentry{214}{\tfrac{8569}{2073600}}{aa bb cd gh}{cd ef fh gh}
\coeffentry{215}{\tfrac{97489}{1382400}}{aa bb cd gh}{cd ef fh hh}
\coeffentry{216}{\tfrac{4579}{165888}}{aa bb cd gh}{cd eg fg hh}
\coeffentry{217}{-\tfrac{1371539}{58060800}}{aa bb cd gh}{cd eg fh fh}
\coeffentry{218}{-\tfrac{38687}{4147200}}{aa bb cd gh}{cd eh fg fh}
\coeffentry{219}{\tfrac{27043}{58060800}}{aa bb cd gh}{cd eh fg gh}
\coeffentry{220}{\tfrac{13}{2419200}}{aa bb cd gh}{cd eh fh fh}
\coeffentry{221}{-\tfrac{18491}{2419200}}{aa bb cd gh}{dd ef ef fh}
\coeffentry{222}{-\tfrac{19147}{2419200}}{aa bb cd gh}{dd ef eh fg}
\coeffentry{223}{\tfrac{1559}{100800}}{aa bb cd gh}{dd ef eh fh}
\coeffentry{224}{-\tfrac{523}{201600}}{aa bb cd gh}{dd ef fg fh}
\coeffentry{225}{\tfrac{499}{403200}}{aa bb cd gh}{dd ef fg hh}
\coeffentry{226}{\tfrac{91}{4320}}{aa bb cd gh}{dd ef fh fh}
\coeffentry{227}{-\tfrac{12463}{806400}}{aa bb cd gh}{dd ef fh gh}
\coeffentry{228}{-\tfrac{157}{89600}}{aa bb cd gh}{dd ef fh hh}
\coeffentry{229}{-\tfrac{10193}{604800}}{aa bb cd gh}{dd eg fg hh}
\coeffentry{230}{\tfrac{13}{15120}}{aa bb cd gh}{dd eg fh fh}
\coeffentry{231}{\tfrac{18761}{1209600}}{aa bb cd gh}{dd eh fg fh}
\coeffentry{232}{\tfrac{13}{4200}}{aa bb cd gh}{dd eh fh fh}
\coeffentry{233}{-\tfrac{19279}{2419200}}{aa bb cd gh}{de ee fg fh}
\coeffentry{234}{\tfrac{97207}{2419200}}{aa bb cd gh}{de ee fg hh}
\coeffentry{235}{\tfrac{5093}{67200}}{aa bb cd gh}{de ee fh fh}
\coeffentry{236}{\tfrac{13007}{1209600}}{aa bb dd ef}{ab cd eg eh}
\coeffentry{237}{-\tfrac{1}{403200}}{aa bb dd ef}{ab cd eg gh}
\coeffentry{238}{-\tfrac{1}{5400}}{aa bb dd ef}{ab cd eh fg}
\coeffentry{239}{-\tfrac{913}{23040}}{aa bb de fg}{cf eg gh hh}
\coeffentry{240}{\tfrac{17743}{2419200}}{aa bb de gh}{cd ce fg fh}
\coeffentry{241}{-\tfrac{41191}{2419200}}{aa bb de gh}{cd ce fg hh}
\coeffentry{242}{-\tfrac{1487}{75600}}{aa bb de gh}{cd ce fh fh}
\coeffentry{243}{\tfrac{253}{37800}}{aa bb de gh}{cd ee fg hh}
\coeffentry{244}{\tfrac{9959}{268800}}{aa bb de gh}{cd ee fh fh}
\coeffentry{245}{-\tfrac{28367}{806400}}{aa bb de gh}{ce ce fh fh}
\coeffentry{246}{-\tfrac{19}{302400}}{aa bb df eg}{cg ch ef hh}
\coeffentry{247}{\tfrac{14603}{1209600}}{aa bb df eg}{cg dh ef hh}
\coeffentry{248}{\tfrac{1375}{48384}}{aa bb df eg}{cg ef gh hh}
\coeffentry{249}{\tfrac{17327}{1209600}}{aa bb df eg}{cg eh fg hh}
\coeffentry{250}{-\tfrac{2101}{86400}}{aa bb df eg}{cg eh fh hh}
\coeffentry{251}{\tfrac{1}{302400}}{aa bb dg ef}{cf cg ch hh}
\coeffentry{252}{\tfrac{5153}{806400}}{aa bb dg ef}{cf cg eh hh}
\coeffentry{253}{-\tfrac{15467}{2419200}}{aa bb dg ef}{cf cg gh hh}
\coeffentry{254}{\tfrac{12241}{806400}}{aa bb dg ef}{cf eg gh hh}
\coeffentry{255}{\tfrac{18313}{1209600}}{aa bb dg ef}{cg ch fh hh}
\coeffentry{256}{-\tfrac{1117}{26880}}{aa bb dg ef}{cg eh fg hh}
\coeffentry{257}{\tfrac{46133}{2419200}}{aa bb dg ef}{cg fg fh hh}
\coeffentry{258}{-\tfrac{1627}{483840}}{aa bb dg ef}{cg fg gh hh}
\coeffentry{259}{-\tfrac{33659}{2419200}}{aa bb dg ef}{cg fh gh hh}
\coeffentry{260}{\tfrac{68569}{2419200}}{aa bb dg ef}{ch eh fg hh}
\coeffentry{261}{-\tfrac{12317}{1209600}}{aa bb dg ef}{ch fg gh hh}
\coeffentry{262}{\tfrac{9377}{1209600}}{aa bb dg ef}{ch fh gh hh}
\coeffentry{263}{\tfrac{11609}{604800}}{aa bb ef gh}{cd ch dg fh}
\coeffentry{264}{\tfrac{2879}{302400}}{aa bb ef gh}{cd dg dh fh}
\coeffentry{265}{-\tfrac{4169}{403200}}{aa bb ef gh}{cd dg eh fh}
\coeffentry{266}{\tfrac{737}{86400}}{aa bb ef gh}{cd dg fg hh}
\coeffentry{267}{\tfrac{569}{483840}}{aa bb ef gh}{cd dg fh hh}
\coeffentry{268}{-\tfrac{106709}{2419200}}{aa bb ef gh}{cd dh fg gh}
\coeffentry{269}{-\tfrac{98869}{2419200}}{aa bb ef gh}{cg df dg hh}
\coeffentry{270}{\tfrac{18539}{483840}}{aa bb ef gh}{cg df dh dh}
\coeffentry{271}{-\tfrac{853}{302400}}{aa bb ef gh}{cg df dh eh}
\coeffentry{272}{\tfrac{4237}{1209600}}{aa bb ef gh}{cg df eg hh}
\coeffentry{273}{\tfrac{47671}{2419200}}{aa bb ef gh}{cg dg fg hh}
\coeffentry{274}{\tfrac{25297}{1209600}}{aa bb ef gh}{cg dg fh hh}
\coeffentry{275}{-\tfrac{1007}{604800}}{aa bb ef gh}{cg dh dh fh}
\coeffentry{276}{\tfrac{2001}{89600}}{aa bb ef gh}{ch df dg gh}
\coeffentry{277}{-\tfrac{331}{483840}}{aa bb ef gh}{ch df eg gh}
\coeffentry{278}{-\tfrac{629}{48384}}{aa bb ef gh}{ch dg dg fh}
\coeffentry{279}{\tfrac{583}{302400}}{aa bb ef gh}{ch dg dh fg}
\coeffentry{280}{\tfrac{41141}{2419200}}{aa bb ef gh}{ch dg dh fh}
\coeffentry{281}{-\tfrac{19}{96768}}{aa bb ef gh}{ch dg eh fg}
\coeffentry{282}{-\tfrac{1}{2400}}{aa bb ef gh}{ch dg eh fh}
\coeffentry{283}{\tfrac{4837}{345600}}{aa bb ef gh}{ch dg fg fh}
\coeffentry{284}{-\tfrac{81761}{2419200}}{aa bb ef gh}{ch dg fg hh}
\coeffentry{285}{-\tfrac{9431}{483840}}{aa bb ef gh}{ch dg fh gh}
\coeffentry{286}{\tfrac{3743}{2419200}}{aa bb ef gh}{ch dh eg fg}
\coeffentry{287}{-\tfrac{407}{241920}}{aa bb ef gh}{ch dh fg gh}
\coeffentry{288}{-\tfrac{20927}{2419200}}{aa bb eg fg}{cd cd fh hh}
\coeffentry{289}{\tfrac{20233}{1209600}}{aa bb eg fg}{cd cd gh hh}
\coeffentry{290}{-\tfrac{253}{48384}}{aa bb eg fg}{cd ch df hh}
\coeffentry{291}{\tfrac{101}{19200}}{aa bb eg fg}{cd ch dg hh}
\coeffentry{292}{\tfrac{3169}{201600}}{aa bb eg fg}{cd dd fh hh}
\coeffentry{293}{-\tfrac{43027}{1209600}}{aa bb eg fg}{cd dd gh hh}
\coeffentry{294}{-\tfrac{367}{60480}}{aa bb eg fg}{cd df dh hh}
\coeffentry{295}{-\tfrac{1051}{403200}}{aa bb eg fg}{cd df eh hh}
\coeffentry{296}{\tfrac{18707}{604800}}{aa bb eg fg}{cd df fh hh}
\coeffentry{297}{-\tfrac{8327}{345600}}{aa bb eg fg}{cd df gh hh}
\coeffentry{298}{\tfrac{293}{50400}}{aa bb eg fg}{cd dg dh hh}
\coeffentry{299}{-\tfrac{23587}{2419200}}{aa bb eg fg}{cd dg fh hh}
\coeffentry{300}{\tfrac{97}{40320}}{aa bb eg fg}{cd dg gh hh}
\coeffentry{301}{-\tfrac{7273}{345600}}{aa bb eg fg}{cd dh fh hh}
\coeffentry{302}{\tfrac{26183}{1209600}}{aa bb eg fg}{cd dh gh hh}
\coeffentry{303}{\tfrac{7249}{2419200}}{aa bb eg fg}{cf de dh hh}
\coeffentry{304}{\tfrac{2747}{806400}}{aa bb eg fg}{cf de fh hh}
\coeffentry{305}{\tfrac{31}{1209600}}{aa bb eg fg}{cf de gh hh}
\coeffentry{306}{\tfrac{2089}{268800}}{aa bb eg fg}{cf df dh hh}
\coeffentry{307}{-\tfrac{20033}{1209600}}{aa bb eg fg}{cf df fh hh}
\coeffentry{308}{\tfrac{1363}{89600}}{aa bb eg fg}{cf df gh hh}
\coeffentry{309}{-\tfrac{1357}{403200}}{aa bb eg fg}{cg dg dh hh}
\coeffentry{310}{\tfrac{421}{1209600}}{aa bb eg fg}{cg dg gh hh}
\coeffentry{311}{-\tfrac{4349}{483840}}{aa bb eg fg}{ch df dh hh}
\coeffentry{312}{\tfrac{571}{120960}}{aa bb eg fg}{ch df eh hh}
\coeffentry{313}{\tfrac{1811}{115200}}{aa bb eg fg}{ch df fh hh}
\coeffentry{314}{\tfrac{4451}{268800}}{aa bb eg fg}{ch dg dh hh}
\coeffentry{315}{-\tfrac{8423}{1209600}}{aa bb eg fg}{ch dg gh hh}
\coeffentry{316}{\tfrac{1}{120960}}{aa bb eg fg}{ch dh fh hh}
\coeffentry{317}{-\tfrac{3281}{241920}}{aa bb eg fg}{ch dh gh hh}
\coeffentry{318}{-\tfrac{57149}{2419200}}{aa bb eg fh}{cd cd dh fg}
\coeffentry{319}{\tfrac{29053}{967680}}{aa bb eg fh}{cd cd fg hh}
\coeffentry{320}{\tfrac{42461}{2419200}}{aa bb eg fh}{cd cd fh gh}
\coeffentry{321}{-\tfrac{407}{345600}}{aa bb eg fh}{cd cg df hh}
\coeffentry{322}{-\tfrac{18317}{1209600}}{aa bb eg fh}{cd ch df dg}
\coeffentry{323}{-\tfrac{7541}{1209600}}{aa bb eg fh}{cd ch dg fh}
\coeffentry{324}{-\tfrac{3509}{1209600}}{aa bb eg fh}{cd ch dh fg}
\coeffentry{325}{-\tfrac{43259}{806400}}{aa bb eg fh}{cd dd fg hh}
\coeffentry{326}{\tfrac{1271}{483840}}{aa bb eg fh}{cd df dg hh}
\coeffentry{327}{\tfrac{6523}{268800}}{aa bb eg fh}{cd dg dh fh}
\coeffentry{328}{-\tfrac{14263}{604800}}{aa bb eg fh}{cd dg fg hh}
\coeffentry{329}{\tfrac{8083}{806400}}{aa bb eg fh}{cd dh fg gh}
\coeffentry{330}{\tfrac{47927}{2419200}}{aa bb eg fh}{cg df dg hh}
\coeffentry{331}{\tfrac{7997}{2419200}}{aa bb eg fh}{cg df dh eh}
\coeffentry{332}{-\tfrac{277}{43200}}{aa bb eg fh}{cg df eh hh}
\coeffentry{333}{-\tfrac{28139}{2419200}}{aa bb eg fh}{cg df fg hh}
\coeffentry{334}{-\tfrac{26141}{1209600}}{aa bb eg fh}{cg df gh hh}
\coeffentry{335}{-\tfrac{4253}{2419200}}{aa bb eg fh}{cg dg fg hh}
\coeffentry{336}{\tfrac{17933}{1209600}}{aa bb eg fh}{ch ch df dg}
\coeffentry{337}{\tfrac{8389}{2419200}}{aa bb eg fh}{ch df dg eh}
\coeffentry{338}{-\tfrac{12479}{604800}}{aa bb eg fh}{ch df dg gh}
\coeffentry{339}{-\tfrac{47269}{2419200}}{aa bb eg fh}{ch df dg hh}
\coeffentry{340}{\tfrac{41}{5600}}{aa bb eg fh}{ch df gh gh}
\coeffentry{341}{-\tfrac{2117}{120960}}{aa bb eg fh}{ch dg dh fg}
\coeffentry{342}{\tfrac{30491}{806400}}{aa bb eg fh}{ch dg fg hh}
\coeffentry{343}{\tfrac{37579}{2419200}}{aa bb eg fh}{ch dg fh gh}
\coeffentry{344}{\tfrac{6901}{604800}}{aa bb eg fh}{ch dh fg fg}
\coeffentry{345}{-\tfrac{13231}{806400}}{aa bb eg fh}{ch dh fg gh}
\coeffentry{346}{\tfrac{19}{50400}}{aa bb eg fh}{ch dh fh gh}
\coeffentry{347}{-\tfrac{1993}{645120}}{aa bb eh fg}{cd cd eg fh}
\coeffentry{348}{-\tfrac{7181}{1209600}}{aa bb eh fg}{cd ch dg ef}
\coeffentry{349}{\tfrac{2999}{201600}}{aa bb eh fg}{cd ch dg fh}
\coeffentry{350}{-\tfrac{3743}{2419200}}{aa bb eh fg}{cd ch dg gh}
\coeffentry{351}{-\tfrac{9377}{1209600}}{aa bb eh fg}{cd df dg dh}
\coeffentry{352}{-\tfrac{2083}{806400}}{aa bb eh fg}{cd dg dh ef}
\coeffentry{353}{-\tfrac{3569}{2419200}}{aa bb eh fg}{cd dg dh fh}
\coeffentry{354}{\tfrac{1679}{151200}}{aa bb eh fg}{cd dg ef hh}
\coeffentry{355}{\tfrac{5827}{2419200}}{aa bb eh fg}{cd dg eh fh}
\coeffentry{356}{\tfrac{22577}{2419200}}{aa bb eh fg}{cd dh eg fh}
\coeffentry{357}{-\tfrac{31}{16128}}{aa bb eh fg}{cd dh fh gh}
\coeffentry{358}{\tfrac{503}{483840}}{aa bb eh fg}{cg df eg hh}
\coeffentry{359}{\tfrac{187}{403200}}{aa bb eh fg}{cg dh dh fh}
\coeffentry{360}{\tfrac{37}{17280}}{aa bb eh fg}{ch dg dh fh}
\coeffentry{361}{-\tfrac{19}{50400}}{aa bb eh fg}{ch dh fh gh}
\coeffentry{362}{-\tfrac{1}{12600}}{aa bb fg gh}{ab cd ce hh}
\coeffentry{363}{-\tfrac{1361}{345600}}{aa bb fg gh}{cd cd ef hh}
\coeffentry{364}{-\tfrac{38287}{2419200}}{aa bb fg gh}{cd cd eg hh}
\coeffentry{365}{-\tfrac{6263}{302400}}{aa bb fg gh}{cd ce de hh}
\coeffentry{366}{\tfrac{15289}{2419200}}{aa bb fg gh}{cd de ee hh}
\coeffentry{367}{\tfrac{5359}{241920}}{aa bb fg gh}{ce de de hh}
\coeffentry{368}{\tfrac{761}{403200}}{aa bb fg gh}{ce de df hh}
\coeffentry{369}{\tfrac{76483}{2419200}}{aa bb fg gh}{ce de dg hh}
\coeffentry{370}{\tfrac{131}{161280}}{aa bb fg gh}{ce de ef hh}
\coeffentry{371}{-\tfrac{11111}{483840}}{aa bb fg gh}{ce de eg hh}
\coeffentry{372}{-\tfrac{607}{268800}}{aa bb fg gh}{cf de ef hh}
\coeffentry{373}{\tfrac{1817}{483840}}{aa bb fg gh}{cf de eg hh}
\coeffentry{374}{-\tfrac{421}{2419200}}{aa bb fg gh}{cf df ef hh}
\coeffentry{375}{\tfrac{341}{604800}}{aa bb fg gh}{cg de eg hh}
\coeffentry{376}{-\tfrac{709}{1209600}}{aa bb fg gh}{cg dg eg hh}
\coeffentry{377}{-\tfrac{1}{12600}}{aa bb fg hh}{ab cd ce gh}
\coeffentry{378}{\tfrac{3883}{4838400}}{aa bb fg hh}{cd cd eg fh}
\coeffentry{379}{\tfrac{1}{12600}}{aa bb fh gh}{ab cd ce fg}
\coeffentry{380}{-\tfrac{5891}{4838400}}{aa bb fh gh}{cd cd ef eg}
\coeffentry{381}{\tfrac{83}{15360}}{aa bb fh gh}{cd cd eg eh}
\coeffentry{382}{\tfrac{10427}{4838400}}{aa bb fh gh}{cd cd eg fg}
\coeffentry{383}{-\tfrac{1597}{2419200}}{aa bb fh gh}{cd cd eg fh}
\coeffentry{384}{\tfrac{77}{34560}}{aa bb fh gh}{cd cd eh eh}
\coeffentry{385}{\tfrac{28451}{1209600}}{aa bb fh gh}{cd ce de eg}
\coeffentry{386}{-\tfrac{23969}{2419200}}{aa bb fh gh}{cd ce de eh}
\coeffentry{387}{\tfrac{20641}{7257600}}{aa bb fh gh}{cd ce de fg}
\coeffentry{388}{\tfrac{1}{18900}}{aa bb fh gh}{cd ce de gh}
\coeffentry{389}{\tfrac{517}{100800}}{aa bb fh gh}{cd dd eh eh}
\coeffentry{390}{\tfrac{559}{403200}}{aa bb fh gh}{cd de ef eg}
\coeffentry{391}{-\tfrac{7291}{1209600}}{aa bb fh gh}{cd de eg eh}
\coeffentry{392}{-\tfrac{41}{4320}}{aa bb fh gh}{cd de eh eh}
\coeffentry{393}{-\tfrac{1363}{161280}}{aa bb fh gh}{ce cg de de}
\coeffentry{394}{\tfrac{1373}{2419200}}{aa bb fh gh}{ce cg de df}
\coeffentry{395}{-\tfrac{1447}{120960}}{aa bb fh gh}{ce cg de dg}
\coeffentry{396}{-\tfrac{7249}{2419200}}{aa bb fh gh}{ce ch de de}
\coeffentry{397}{\tfrac{1109}{172800}}{aa bb fh gh}{ce ch de dg}
\coeffentry{398}{\tfrac{3251}{1209600}}{aa bb fh gh}{ce ch de dh}
\coeffentry{399}{-\tfrac{571}{120960}}{aa bb fh gh}{ce de de fg}
\coeffentry{400}{-\tfrac{7261}{2419200}}{aa bb fh gh}{ce de dg ef}
\coeffentry{401}{-\tfrac{1}{1575}}{aa bb fh gh}{ce de dg fg}
\coeffentry{402}{-\tfrac{4021}{604800}}{aa bb fh gh}{ce de dh eg}
\coeffentry{403}{-\tfrac{7493}{2419200}}{aa bb fh gh}{ce de dh fg}
\coeffentry{404}{\tfrac{397}{1209600}}{aa bb fh gh}{ce de ef eg}
\coeffentry{405}{\tfrac{317}{48384}}{aa bb fh gh}{ce de eg eh}
\coeffentry{406}{\tfrac{479}{268800}}{aa bb fh gh}{ce de eg fg}
\coeffentry{407}{\tfrac{551}{80640}}{aa bb fh gh}{ce de eh eh}
\coeffentry{408}{\tfrac{7109}{2419200}}{aa bb fh gh}{ce de eh fg}
\coeffentry{409}{\tfrac{1373}{806400}}{aa bb fh gh}{ce dg dg ef}
\coeffentry{410}{\tfrac{11}{60480}}{aa bb fh gh}{ce dh dh eh}
\coeffentry{411}{-\tfrac{8237}{2419200}}{aa bb fh gh}{cf de dg eg}
\coeffentry{412}{-\tfrac{503}{604800}}{aa bb fh gh}{cg de dg eg}
\coeffentry{413}{\tfrac{1}{11520}}{aa bc bc dd}{cd ef eg eh}
\coeffentry{414}{-\tfrac{1}{5040}}{aa bc bc dd}{cd ef eg gh}
\coeffentry{415}{\tfrac{1}{4725}}{aa bc bc de}{bc df dg gh}
\coeffentry{416}{\tfrac{1}{3780}}{aa bc bc de}{bc dg dh ef}
\coeffentry{417}{\tfrac{17}{37800}}{aa bc bc ee}{bc de fg fh}
\coeffentry{418}{-\tfrac{1}{1890}}{aa bc bc ee}{cc de fg fh}
\coeffentry{419}{-\tfrac{7153}{7257600}}{aa bc bc ee}{de fg fh gh}
\coeffentry{420}{\tfrac{2623}{1209600}}{aa bc bc ee}{de fg gh hh}
\coeffentry{421}{\tfrac{23}{60480}}{aa bc bc ee}{de fh gh gh}
\coeffentry{422}{\tfrac{1}{21600}}{aa bc bc fg}{de dh eg hh}
\coeffentry{423}{-\tfrac{73}{1209600}}{aa bc bc fg}{de eg eh hh}
\coeffentry{424}{-\tfrac{8011}{1209600}}{aa bc bc fg}{de eg fh hh}
\coeffentry{425}{\tfrac{9299}{2419200}}{aa bc bc fg}{de eg gh hh}
\coeffentry{426}{-\tfrac{23}{25200}}{aa bc bc fg}{de eh gh hh}
\coeffentry{427}{\tfrac{229}{115200}}{aa bc bc fg}{dg ef eh hh}
\coeffentry{428}{-\tfrac{2039}{2419200}}{aa bc bc fg}{dg ef gh hh}
\coeffentry{429}{-\tfrac{221}{604800}}{aa bc bc fg}{dg eg eh hh}
\coeffentry{430}{-\tfrac{29}{60480}}{aa bc bc fg}{dg eg gh hh}
\coeffentry{431}{-\tfrac{29}{28800}}{aa bc bc fg}{dh eg eh hh}
\coeffentry{432}{\tfrac{161}{69120}}{aa bc bc fg}{dh eg fh hh}
\coeffentry{433}{\tfrac{67}{86400}}{aa bc bc fg}{dh eg gh hh}
\coeffentry{434}{\tfrac{323}{302400}}{aa bc bc fg}{dh eh gh hh}
\coeffentry{435}{\tfrac{101}{302400}}{aa bc bc gh}{de df ef fh}
\coeffentry{436}{-\tfrac{13}{48384}}{aa bc bc gh}{de ef fg fh}
\coeffentry{437}{\tfrac{77}{57600}}{aa bc bc gh}{de ef fh fh}
\coeffentry{438}{\tfrac{461}{1209600}}{aa bc bc gh}{df dh ef ef}
\coeffentry{439}{-\tfrac{517}{1209600}}{aa bc bc gh}{df dh ef eg}
\coeffentry{440}{-\tfrac{121}{302400}}{aa bc bc gh}{df dh ef eh}
\coeffentry{441}{-\tfrac{73}{151200}}{aa bc bc gh}{df ef ef eh}
\coeffentry{442}{\tfrac{8989}{2419200}}{aa bc bc gh}{df ef eg hh}
\coeffentry{443}{\tfrac{157}{241920}}{aa bc bc gh}{df ef eh fg}
\coeffentry{444}{\tfrac{67}{403200}}{aa bc bc gh}{df ef eh fh}
\coeffentry{445}{\tfrac{103}{241920}}{aa bc bc gh}{df ef eh gh}
\coeffentry{446}{-\tfrac{661}{161280}}{aa bc bc gh}{df ef eh hh}
\coeffentry{447}{\tfrac{43}{604800}}{aa bc bc gh}{df ef fg fh}
\coeffentry{448}{-\tfrac{23}{86400}}{aa bc bc gh}{df ef fg hh}
\coeffentry{449}{-\tfrac{641}{604800}}{aa bc bc gh}{df ef fh fh}
\coeffentry{450}{-\tfrac{23}{60480}}{aa bc bc gh}{df ef fh gh}
\coeffentry{451}{\tfrac{649}{1209600}}{aa bc bc gh}{df ef fh hh}
\coeffentry{452}{\tfrac{13}{604800}}{aa bc bc gh}{df eg eh fh}
\coeffentry{453}{\tfrac{8147}{2419200}}{aa bc bc gh}{df eh eh fg}
\coeffentry{454}{-\tfrac{487}{604800}}{aa bc bc gh}{df eh eh fh}
\coeffentry{455}{\tfrac{9}{12800}}{aa bc bc gh}{dg ef eh fh}
\coeffentry{456}{-\tfrac{8657}{2419200}}{aa bc bc gh}{dg ef fg hh}
\coeffentry{457}{-\tfrac{1753}{2419200}}{aa bc bc gh}{dg ef fh fh}
\coeffentry{458}{-\tfrac{1}{21600}}{aa bc bc gh}{dg eg fg hh}
\coeffentry{459}{-\tfrac{367}{172800}}{aa bc bc gh}{dh ef ef fg}
\coeffentry{460}{\tfrac{251}{403200}}{aa bc bc gh}{dh ef eh fh}
\coeffentry{461}{-\tfrac{1}{33600}}{aa bc bd cd}{cd ef eg eh}
\coeffentry{462}{-\tfrac{1}{50400}}{aa bc bd cd}{cd ef eg gh}
\coeffentry{463}{\tfrac{1}{9450}}{aa bc bd cd}{dd ef eg eh}
\coeffentry{464}{-\tfrac{1}{5600}}{aa bc bd cd}{dd ef eg gh}
\coeffentry{465}{-\tfrac{263}{29030400}}{aa bc bd cd}{eg eh fg fh}
\coeffentry{466}{\tfrac{419}{1209600}}{aa bc bd cd}{eg fg fh hh}
\coeffentry{467}{\tfrac{1471}{3628800}}{aa bc bd cd}{eg fg gh hh}
\coeffentry{468}{\tfrac{499}{7257600}}{aa bc bd cd}{eg fh fh gh}
\coeffentry{469}{-\tfrac{1423}{3628800}}{aa bc bd cd}{eh fg fh gh}
\coeffentry{470}{\tfrac{349}{2073600}}{aa bc bd cd}{eh fg gh gh}
\coeffentry{471}{-\tfrac{23}{151200}}{aa bc bd dd}{cc ef eg eh}
\coeffentry{472}{\tfrac{11}{50400}}{aa bc bd dd}{cc ef eg gh}
\coeffentry{473}{\tfrac{1}{6300}}{aa bc bd ef}{eh fg gg hh}
\coeffentry{474}{-\tfrac{1}{6300}}{aa bc bd ef}{fg fh gg hh}
\coeffentry{475}{-\tfrac{1}{3150}}{aa bc bd ef}{fh gg gh hh}
\coeffentry{476}{\tfrac{1}{6300}}{aa bc bd fg}{ef eg eh hh}
\coeffentry{477}{-\tfrac{1}{6300}}{aa bc bd fg}{ef gh gh hh}
\coeffentry{478}{\tfrac{1}{6300}}{aa bc bd fg}{eg eg eh hh}
\coeffentry{479}{-\tfrac{1}{6300}}{aa bc bd fg}{eg eh fg hh}
\coeffentry{480}{\tfrac{1}{6300}}{aa bc bd fg}{eg fh hh hh}
\coeffentry{481}{\tfrac{1}{6300}}{aa bc bd fg}{eh fh gh hh}
\coeffentry{482}{-\tfrac{1}{6300}}{aa bc bd fh}{ef eg gg hh}
\coeffentry{483}{\tfrac{1}{6300}}{aa bc bd fh}{ef fg gg hh}
\coeffentry{484}{-\tfrac{1}{6300}}{aa bc bd fh}{eg fg gg hh}
\coeffentry{485}{\tfrac{1}{6300}}{aa bc bd fh}{eg gg gh hh}
\coeffentry{486}{-\tfrac{1}{6300}}{aa bc bd gh}{ef fg fh fh}
\coeffentry{487}{\tfrac{1}{6300}}{aa bc bd gh}{ef fg gh hh}
\coeffentry{488}{-\tfrac{1}{16800}}{aa bc cd dd}{bc ef eg eh}
\coeffentry{489}{\tfrac{1}{10080}}{aa bc cd dd}{bc ef eg gh}
\coeffentry{490}{\tfrac{121}{67200}}{aa bc cd fg}{dd eg eh hh}
\coeffentry{491}{-\tfrac{583}{57600}}{aa bc cd fg}{dd eg fh hh}
\coeffentry{492}{\tfrac{131}{57600}}{aa bc cd fg}{dd eg gh hh}
\coeffentry{493}{-\tfrac{701}{241920}}{aa bc cd fg}{dd eh gh hh}
\coeffentry{494}{-\tfrac{359}{100800}}{aa bc cd gh}{dd ef ef fh}
\coeffentry{495}{\tfrac{97}{134400}}{aa bc cd gh}{dd ef eh fg}
\coeffentry{496}{-\tfrac{349}{1209600}}{aa bc cd gh}{dd ef eh fh}
\coeffentry{497}{-\tfrac{1}{100800}}{aa bc cd gh}{dd ef fg fh}
\coeffentry{498}{\tfrac{1541}{302400}}{aa bc cd gh}{dd ef fg hh}
\coeffentry{499}{\tfrac{1381}{151200}}{aa bc cd gh}{dd ef fh fh}
\coeffentry{500}{\tfrac{421}{604800}}{aa bc cd gh}{dd ef fh gh}
\coeffentry{501}{-\tfrac{15307}{1209600}}{aa bc cd gh}{dd ef fh hh}
\coeffentry{502}{-\tfrac{1657}{1209600}}{aa bc cd gh}{dd eg fg hh}
\coeffentry{503}{\tfrac{913}{241920}}{aa bc cd gh}{dd eg fh fh}
\coeffentry{504}{-\tfrac{271}{241920}}{aa bc cd gh}{dd eh fg fh}
\coeffentry{505}{\tfrac{313}{604800}}{aa bc cd gh}{dd eh fh fh}
\coeffentry{506}{\tfrac{29}{9450}}{aa bc dd fg}{cd eg eh hh}
\coeffentry{507}{-\tfrac{5561}{604800}}{aa bc dd fg}{cd eg fh hh}
\coeffentry{508}{\tfrac{1283}{134400}}{aa bc dd fg}{cd eg gh hh}
\coeffentry{509}{-\tfrac{367}{37800}}{aa bc dd fg}{cd eh gh hh}
\coeffentry{510}{\tfrac{2003}{604800}}{aa bc dd gh}{cd ef ef fh}
\coeffentry{511}{-\tfrac{221}{604800}}{aa bc dd gh}{cd ef eh fg}
\coeffentry{512}{-\tfrac{1819}{302400}}{aa bc dd gh}{cd ef eh fh}
\coeffentry{513}{\tfrac{6373}{1209600}}{aa bc dd gh}{cd ef fg hh}
\coeffentry{514}{\tfrac{3869}{604800}}{aa bc dd gh}{cd ef fh fh}
\coeffentry{515}{\tfrac{221}{604800}}{aa bc dd gh}{cd ef fh gh}
\coeffentry{516}{-\tfrac{23}{5376}}{aa bc dd gh}{cd ef fh hh}
\coeffentry{517}{-\tfrac{4271}{1209600}}{aa bc dd gh}{cd eg fg hh}
\coeffentry{518}{\tfrac{2431}{403200}}{aa bc dd gh}{cd eg fh fh}
\coeffentry{519}{\tfrac{1}{12600}}{aa bc de de}{bc de fg fh}
\coeffentry{520}{-\tfrac{219421}{3628800}}{aa bc de fg}{bc ee gh hh}
\coeffentry{521}{\tfrac{24589}{483840}}{aa bc de fg}{cc ee gh hh}
\coeffentry{522}{\tfrac{13043}{907200}}{aa bc de gh}{bc ee fg fh}
\coeffentry{523}{\tfrac{2959}{907200}}{aa bc de gh}{bc ee fg hh}
\coeffentry{524}{-\tfrac{36227}{1814400}}{aa bc de gh}{bc ee fh fh}
\coeffentry{525}{\tfrac{883}{1209600}}{aa bc de gh}{cc ee fg fh}
\coeffentry{526}{-\tfrac{1417}{120960}}{aa bc de gh}{cc ee fg hh}
\coeffentry{527}{-\tfrac{523}{120960}}{aa bc de gh}{cc ee fh fh}
\coeffentry{528}{-\tfrac{41381}{518400}}{aa bc dg ef}{bc eh fg hh}
\coeffentry{529}{\tfrac{3529}{86400}}{aa bc dg ef}{bc fg gh hh}
\coeffentry{530}{\tfrac{1331}{302400}}{aa bc dg ef}{bc fh gh hh}
\coeffentry{531}{\tfrac{389}{12096}}{aa bc dg ef}{cc eh fg hh}
\coeffentry{532}{-\tfrac{14717}{1209600}}{aa bc dg ef}{cc fg gh hh}
\coeffentry{533}{-\tfrac{2299}{403200}}{aa bc dg ef}{cc fh gh hh}
\coeffentry{534}{\tfrac{4627}{172800}}{aa bc ef gh}{bc dg dh fh}
\coeffentry{535}{-\tfrac{8789}{1209600}}{aa bc ef gh}{bc dg eh fh}
\coeffentry{536}{-\tfrac{5951}{86400}}{aa bc ef gh}{bc dg fg hh}
\coeffentry{537}{-\tfrac{3473}{302400}}{aa bc ef gh}{bc dg fh hh}
\coeffentry{538}{-\tfrac{8311}{518400}}{aa bc ef gh}{bc dh fg gh}
\coeffentry{539}{-\tfrac{2623}{134400}}{aa bc ef gh}{cc dg dh fh}
\coeffentry{540}{\tfrac{23}{100800}}{aa bc ef gh}{cc dg eh fh}
\coeffentry{541}{\tfrac{1453}{40320}}{aa bc ef gh}{cc dg fg hh}
\coeffentry{542}{\tfrac{2353}{80640}}{aa bc ef gh}{cc dg fh hh}
\coeffentry{543}{\tfrac{557}{1209600}}{aa bc ef gh}{cc dh fg gh}
\coeffentry{544}{-\tfrac{89}{134400}}{aa bc eg fg}{bc df dh hh}
\coeffentry{545}{-\tfrac{1667}{302400}}{aa bc eg fg}{bc df eh hh}
\coeffentry{546}{-\tfrac{67}{241920}}{aa bc eg fg}{bc df fh hh}
\coeffentry{547}{\tfrac{71}{17280}}{aa bc eg fg}{bc df gh hh}
\coeffentry{548}{-\tfrac{11}{37800}}{aa bc eg fg}{bc dg dh hh}
\coeffentry{549}{\tfrac{473}{201600}}{aa bc eg fg}{bc dg gh hh}
\coeffentry{550}{\tfrac{221}{403200}}{aa bc eg fg}{bc dh fh hh}
\coeffentry{551}{\tfrac{2161}{201600}}{aa bc eg fg}{cc df dh hh}
\coeffentry{552}{\tfrac{2927}{1209600}}{aa bc eg fg}{cc df eh hh}
\coeffentry{553}{-\tfrac{1349}{403200}}{aa bc eg fg}{cc df fh hh}
\coeffentry{554}{-\tfrac{461}{403200}}{aa bc eg fg}{cc df gh hh}
\coeffentry{555}{-\tfrac{83}{34560}}{aa bc eg fg}{cc dg dh hh}
\coeffentry{556}{\tfrac{43}{11520}}{aa bc eg fg}{cc dg gh hh}
\coeffentry{557}{\tfrac{2593}{1209600}}{aa bc eg fg}{cc dh fh hh}
\coeffentry{558}{\tfrac{2549}{1209600}}{aa bc eg fg}{cc dh gh hh}
\coeffentry{559}{\tfrac{11}{24192}}{aa bc eg fg}{cd dd fh hh}
\coeffentry{560}{-\tfrac{1087}{1209600}}{aa bc eg fg}{cd dd gh hh}
\coeffentry{561}{\tfrac{3529}{67200}}{aa bc eg fh}{bc df dg hh}
\coeffentry{562}{-\tfrac{5981}{518400}}{aa bc eg fh}{bc dg dh fh}
\coeffentry{563}{\tfrac{685}{72576}}{aa bc eg fh}{bc dg fg hh}
\coeffentry{564}{\tfrac{4999}{403200}}{aa bc eg fh}{bc dh dh fg}
\coeffentry{565}{\tfrac{58741}{3628800}}{aa bc eg fh}{bc dh fg gh}
\coeffentry{566}{\tfrac{23179}{725760}}{aa bc eg fh}{bc dh fg hh}
\coeffentry{567}{\tfrac{397}{80640}}{aa bc eg fh}{bc dh fh gh}
\coeffentry{568}{-\tfrac{59377}{1209600}}{aa bc eg fh}{cc df dg hh}
\coeffentry{569}{\tfrac{241}{11520}}{aa bc eg fh}{cc dg dh fh}
\coeffentry{570}{\tfrac{11}{201600}}{aa bc eg fh}{cc dg fg hh}
\coeffentry{571}{\tfrac{3077}{1209600}}{aa bc eg fh}{cc dh dh fg}
\coeffentry{572}{-\tfrac{16091}{1209600}}{aa bc eg fh}{cc dh fg gh}
\coeffentry{573}{-\tfrac{34207}{1209600}}{aa bc eg fh}{cc dh fg hh}
\coeffentry{574}{-\tfrac{1063}{172800}}{aa bc eg fh}{cc dh fh gh}
\coeffentry{575}{-\tfrac{1}{160}}{aa bc eg fh}{cd dd fg hh}
\coeffentry{576}{\tfrac{8861}{302400}}{aa bc eg fh}{cd dd fh gh}
\coeffentry{577}{-\tfrac{379}{75600}}{aa bc eh fg}{bc df dg dh}
\coeffentry{578}{\tfrac{34213}{3628800}}{aa bc eh fg}{bc df gh gh}
\coeffentry{579}{\tfrac{14183}{3628800}}{aa bc eh fg}{bc dg dh ef}
\coeffentry{580}{-\tfrac{209}{6300}}{aa bc eh fg}{bc dg dh fh}
\coeffentry{581}{\tfrac{1103}{129600}}{aa bc eh fg}{bc dg ef hh}
\coeffentry{582}{\tfrac{4433}{1814400}}{aa bc eh fg}{bc dg eh fh}
\coeffentry{583}{\tfrac{1763}{241920}}{aa bc eh fg}{cc df dg dh}
\coeffentry{584}{-\tfrac{9167}{604800}}{aa bc eh fg}{cc df gh gh}
\coeffentry{585}{-\tfrac{3667}{302400}}{aa bc eh fg}{cc dg dh ef}
\coeffentry{586}{\tfrac{16847}{604800}}{aa bc eh fg}{cc dg dh fh}
\coeffentry{587}{\tfrac{5041}{403200}}{aa bc eh fg}{cc dg ef hh}
\coeffentry{588}{\tfrac{437}{120960}}{aa bc eh fg}{cc dg eh fh}
\coeffentry{589}{\tfrac{1751}{604800}}{aa bc eh fg}{cc dh fh gh}
\coeffentry{590}{-\tfrac{3727}{1209600}}{aa bc eh fg}{cd dd eg fh}
\coeffentry{591}{\tfrac{11561}{604800}}{aa bc eh fg}{cd dd fh gh}
\coeffentry{592}{-\tfrac{6269}{345600}}{aa bc eh fg}{cd dd gh gh}
\coeffentry{593}{\tfrac{19}{60480}}{aa bc fg gh}{bc de ef hh}
\coeffentry{594}{-\tfrac{5989}{1209600}}{aa bc fg gh}{cc de ef hh}
\coeffentry{595}{\tfrac{899}{75600}}{aa bc fg gh}{cc de eg hh}
\coeffentry{596}{\tfrac{9343}{1209600}}{aa bc fg gh}{cc df ef hh}
\coeffentry{597}{-\tfrac{1}{403200}}{aa bc fg gh}{cc dg eg hh}
\coeffentry{598}{\tfrac{61}{15120}}{aa bc fg gh}{cd dd ef hh}
\coeffentry{599}{\tfrac{47}{201600}}{aa bc fg gh}{cd dd eg hh}
\coeffentry{600}{\tfrac{17}{40320}}{aa bc fh gh}{bc de dg ef}
\coeffentry{601}{\tfrac{323}{604800}}{aa bc fh gh}{bc de dh eg}
\coeffentry{602}{-\tfrac{19}{30240}}{aa bc fh gh}{bc de eg fg}
\coeffentry{603}{-\tfrac{1}{4200}}{aa bc fh gh}{bc de eg fh}
\coeffentry{604}{-\tfrac{257}{60480}}{aa bc fh gh}{cc de de eg}
\coeffentry{605}{-\tfrac{437}{201600}}{aa bc fh gh}{cc de dg ef}
\coeffentry{606}{-\tfrac{181}{1209600}}{aa bc fh gh}{cc de dh eg}
\coeffentry{607}{-\tfrac{341}{151200}}{aa bc fh gh}{cc de dh eh}
\coeffentry{608}{-\tfrac{193}{302400}}{aa bc fh gh}{cc de ef eg}
\coeffentry{609}{\tfrac{1537}{604800}}{aa bc fh gh}{cc de eg fg}
\coeffentry{610}{\tfrac{797}{403200}}{aa bc fh gh}{cc de eg fh}
\coeffentry{611}{-\tfrac{73}{604800}}{aa bc fh gh}{cc dh eh eh}
\coeffentry{612}{\tfrac{317}{120960}}{aa bc fh gh}{cd dd ef eg}
\coeffentry{613}{-\tfrac{19}{26880}}{aa bc fh gh}{cd dd eg eh}
\coeffentry{614}{-\tfrac{5987}{1209600}}{aa bc fh gh}{cd dd eg fg}
\coeffentry{615}{\tfrac{1}{768}}{aa bc fh gh}{cd dd eg fh}
\coeffentry{616}{-\tfrac{11}{50400}}{aa bc fh gh}{cd dd eh eh}
\coeffentry{617}{-\tfrac{1}{21600}}{aa bd cd cd}{bc ef eg eh}
\coeffentry{618}{\tfrac{11}{50400}}{aa bd cd cd}{bc ef eg gh}
\coeffentry{619}{-\tfrac{1}{21600}}{aa bd cd cd}{bd ef eg eh}
\coeffentry{620}{-\tfrac{1}{50400}}{aa bd cd cd}{bd ef eg gh}
\coeffentry{621}{\tfrac{701}{2419200}}{aa bd cd cd}{eg eh fg fh}
\coeffentry{622}{-\tfrac{701}{604800}}{aa bd cd cd}{eg fg fh hh}
\coeffentry{623}{-\tfrac{1}{134400}}{aa bd cd cd}{eg fg gh hh}
\coeffentry{624}{-\tfrac{1}{134400}}{aa bd cd cd}{eg fh fh gh}
\coeffentry{625}{\tfrac{1}{134400}}{aa bd cd cd}{eh fg fh gh}
\coeffentry{626}{\tfrac{11}{10080}}{aa bd cd fg}{bc eg eh hh}
\coeffentry{627}{\tfrac{1}{151200}}{aa bd cd fg}{bc eg fh hh}
\coeffentry{628}{\tfrac{313}{241920}}{aa bd cd fg}{bc eh gh hh}
\coeffentry{629}{-\tfrac{2069}{1209600}}{aa bd cd fg}{cd eg eh hh}
\coeffentry{630}{\tfrac{469}{172800}}{aa bd cd fg}{cd eg fh hh}
\coeffentry{631}{-\tfrac{337}{151200}}{aa bd cd fg}{cd eg gh hh}
\coeffentry{632}{\tfrac{1429}{1209600}}{aa bd cd fg}{cd eh gh hh}
\coeffentry{633}{-\tfrac{1}{151200}}{aa bd cd gh}{bc ef eh fg}
\coeffentry{634}{-\tfrac{1}{43200}}{aa bd cd gh}{bc ef fg fh}
\coeffentry{635}{-\tfrac{1399}{604800}}{aa bd cd gh}{bc ef fg hh}
\coeffentry{636}{\tfrac{2069}{1209600}}{aa bd cd gh}{cd ef eh fh}
\coeffentry{637}{-\tfrac{469}{172800}}{aa bd cd gh}{cd ef fg hh}
\coeffentry{638}{-\tfrac{1429}{1209600}}{aa bd cd gh}{cd ef fh fh}
\coeffentry{639}{\tfrac{1}{1890}}{aa bd cd gh}{cd eg fg hh}
\coeffentry{640}{-\tfrac{13}{1209600}}{aa bd cd gh}{cd eh fh fh}
\coeffentry{641}{\tfrac{61}{17280}}{aa bf ce dg}{df eg gh hh}
\coeffentry{642}{-\tfrac{23}{172800}}{aa bf ce dg}{dh eg fg hh}
\coeffentry{643}{\tfrac{2539}{403200}}{aa bg cf de}{ch dg ef hh}
\coeffentry{644}{-\tfrac{1453}{172800}}{aa bg cf de}{dg ef gh hh}
\coeffentry{645}{\tfrac{3677}{241920}}{aa bg cf de}{dh eg fg hh}
\coeffentry{646}{-\tfrac{739}{86400}}{aa bg cf de}{dh eg fh hh}
\coeffentry{647}{-\tfrac{307}{24192}}{aa bg cf de}{eg fg fh hh}
\coeffentry{648}{-\tfrac{1}{640}}{aa bg cf de}{eg fg gh hh}
\coeffentry{649}{\tfrac{59}{30240}}{aa bg cf de}{eh fg gh hh}
\coeffentry{650}{-\tfrac{589}{604800}}{aa bg cf de}{eh fh gh hh}
\coeffentry{651}{\tfrac{1}{21600}}{aa cd cd cd}{bd ef eg eh}
\coeffentry{652}{\tfrac{1}{50400}}{aa cd cd cd}{bd ef eg gh}
\coeffentry{653}{\tfrac{1121}{302400}}{aa cd eg fg}{bc bd fh hh}
\coeffentry{654}{-\tfrac{187}{60480}}{aa cd eg fg}{bc bd gh hh}
\coeffentry{655}{-\tfrac{1223}{302400}}{aa cd eg fg}{bc dd fh hh}
\coeffentry{656}{\tfrac{421}{1209600}}{aa cd eg fg}{bc dd gh hh}
\coeffentry{657}{-\tfrac{127}{302400}}{aa cd eg fg}{bd bd fh hh}
\coeffentry{658}{\tfrac{421}{120960}}{aa cd eg fg}{bd bd gh hh}
\coeffentry{659}{-\tfrac{83}{75600}}{aa cd eg fg}{be df fh hh}
\coeffentry{660}{\tfrac{1699}{241920}}{aa cd eg fh}{bc bd fg hh}
\coeffentry{661}{-\tfrac{247}{80640}}{aa cd eg fh}{bc bd fh gh}
\coeffentry{662}{-\tfrac{83}{50400}}{aa cd eg fh}{bc dd fg hh}
\coeffentry{663}{\tfrac{341}{100800}}{aa cd eg fh}{bc dd fh gh}
\coeffentry{664}{-\tfrac{3487}{1209600}}{aa cd eg fh}{bd bd fg hh}
\coeffentry{665}{-\tfrac{547}{57600}}{aa cd eg fh}{bd bd fh gh}
\coeffentry{666}{\tfrac{811}{1209600}}{aa cd eh fg}{bc bd eg fh}
\coeffentry{667}{-\tfrac{197}{21600}}{aa cd eh fg}{bc bd fh gh}
\coeffentry{668}{\tfrac{529}{80640}}{aa cd eh fg}{bc bd gh gh}
\coeffentry{669}{-\tfrac{79}{89600}}{aa cd eh fg}{bc dd eg fh}
\coeffentry{670}{\tfrac{5171}{201600}}{aa cd eh fg}{bc dd fh gh}
\coeffentry{671}{-\tfrac{10937}{1209600}}{aa cd eh fg}{bc dd gh gh}
\coeffentry{672}{\tfrac{463}{1209600}}{aa cd eh fg}{bd bd eg fh}
\coeffentry{673}{-\tfrac{29}{403200}}{aa cd eh fg}{bd bd fh gh}
\coeffentry{674}{\tfrac{207}{22400}}{aa cd eh fg}{bd bd gh gh}
\coeffentry{675}{-\tfrac{2021}{1209600}}{aa cd fg gh}{bc bd ef hh}
\coeffentry{676}{\tfrac{3347}{1209600}}{aa cd fg gh}{bc bd eg hh}
\coeffentry{677}{-\tfrac{4849}{1209600}}{aa cd fg gh}{bc dd ef hh}
\coeffentry{678}{-\tfrac{857}{1209600}}{aa cd fg gh}{bd bd ef hh}
\coeffentry{679}{-\tfrac{109}{37800}}{aa cd fg gh}{bd bd eg hh}
\coeffentry{680}{-\tfrac{3091}{1209600}}{aa cd fh gh}{bc bd eg eh}
\coeffentry{681}{-\tfrac{1}{25200}}{aa cd fh gh}{bc bd eh eh}
\coeffentry{682}{\tfrac{47}{12096}}{aa cd fh gh}{bc dd eg eh}
\coeffentry{683}{-\tfrac{1}{1260}}{aa cd fh gh}{bd bd eg eh}
\coeffentry{684}{\tfrac{47}{8400}}{aa ce df gh}{bf dg eg hh}
\coeffentry{685}{\tfrac{3851}{403200}}{aa ce df gh}{bf dh eg gh}
\coeffentry{686}{-\tfrac{1289}{201600}}{aa ce dg fg}{bf bh de hh}
\coeffentry{687}{-\tfrac{31}{134400}}{aa ce dg fg}{bf de fh hh}
\coeffentry{688}{-\tfrac{59}{43200}}{aa ce dg fg}{bf de gh hh}
\coeffentry{689}{-\tfrac{131}{15120}}{aa ce dg fg}{bf dh ef hh}
\coeffentry{690}{\tfrac{53}{604800}}{aa ce dg fg}{bf dh eg hh}
\coeffentry{691}{\tfrac{1597}{1209600}}{aa ce dg fg}{bf dh eh hh}
\coeffentry{692}{-\tfrac{71}{44800}}{aa ce dg fh}{bf bg de hh}
\coeffentry{693}{-\tfrac{193}{151200}}{aa ce dg fh}{bf bg dh eh}
\coeffentry{694}{\tfrac{1123}{1209600}}{aa ce dh fg}{bf bg bh de}
\coeffentry{695}{\tfrac{4673}{302400}}{aa ce dh fg}{bf de gh gh}
\coeffentry{696}{\tfrac{893}{80640}}{aa ce dh fg}{bf dg eg hh}
\coeffentry{697}{-\tfrac{1447}{604800}}{aa ce dh fg}{bf dg eh gh}
\coeffentry{698}{-\tfrac{919}{604800}}{aa ce dh fg}{bf dh eg gh}
\coeffentry{699}{-\tfrac{19}{16128}}{aa cf de gh}{bf bg eg hh}
\coeffentry{700}{-\tfrac{167}{12096}}{aa cf de gh}{bf bh eg gh}
\coeffentry{701}{-\tfrac{43}{10080}}{aa cf de gh}{bg bh dh ef}
\coeffentry{702}{\tfrac{3583}{241920}}{aa cf de gh}{bg bh eh fg}
\coeffentry{703}{-\tfrac{11}{3456}}{aa cf de gh}{bg bh eh fh}
\coeffentry{704}{\tfrac{247}{201600}}{aa cf de gh}{bg ch dh ef}
\coeffentry{705}{\tfrac{359}{100800}}{aa cf de gh}{bg dg ef hh}
\coeffentry{706}{-\tfrac{2177}{172800}}{aa cf de gh}{bg dh ef hh}
\coeffentry{707}{\tfrac{2813}{1209600}}{aa cf de gh}{bg dh eh fh}
\coeffentry{708}{\tfrac{17791}{1209600}}{aa cf de gh}{bg ef fg hh}
\coeffentry{709}{\tfrac{709}{172800}}{aa cf de gh}{bg eg fg hh}
\coeffentry{710}{-\tfrac{101}{7560}}{aa cf de gh}{bg eh fg hh}
\coeffentry{711}{-\tfrac{37}{75600}}{aa cf de gh}{bg eh fh gh}
\coeffentry{712}{\tfrac{47309}{1209600}}{aa cf de gh}{bg eh fh hh}
\coeffentry{713}{\tfrac{8081}{1209600}}{aa cf de gh}{bh bh eg fg}
\coeffentry{714}{-\tfrac{3223}{241920}}{aa cf de gh}{bh dg ef gh}
\coeffentry{715}{\tfrac{181}{22400}}{aa cf de gh}{bh dg eg fh}
\coeffentry{716}{-\tfrac{613}{151200}}{aa cf de gh}{bh dh eg fg}
\coeffentry{717}{-\tfrac{11849}{1209600}}{aa cf de gh}{bh eg fg hh}
\coeffentry{718}{-\tfrac{3487}{1209600}}{aa cf de gh}{bh eh fg gh}
\coeffentry{719}{-\tfrac{173}{50400}}{aa cf dg eg}{be bf dh hh}
\coeffentry{720}{\tfrac{1657}{604800}}{aa cf dg eg}{be df fh hh}
\coeffentry{721}{\tfrac{449}{1209600}}{aa cf dg eg}{bg dh ef hh}
\coeffentry{722}{-\tfrac{377}{57600}}{aa cf dg eg}{bh dh ef hh}
\coeffentry{723}{-\tfrac{289}{43200}}{aa cf dg eh}{bh dg ef gh}
\coeffentry{724}{-\tfrac{3901}{403200}}{aa cf dh eg}{bg bh de fh}
\coeffentry{725}{-\tfrac{3223}{302400}}{aa cf dh eg}{bg bh df eh}
\coeffentry{726}{\tfrac{8791}{302400}}{aa cf dh eg}{bg bh dh ef}
\coeffentry{727}{\tfrac{619}{86400}}{aa cf dh eg}{bg de fg hh}
\coeffentry{728}{-\tfrac{11929}{1209600}}{aa cf dh eg}{bg de fh hh}
\coeffentry{729}{-\tfrac{8461}{1209600}}{aa cf dh eg}{bg dg ef hh}
\coeffentry{730}{-\tfrac{761}{604800}}{aa cf dh eg}{bg dh eh fh}
\coeffentry{731}{\tfrac{11951}{1209600}}{aa cf dh eg}{bh de fg gh}
\coeffentry{732}{\tfrac{193}{48384}}{aa cf dh eg}{bh dg ef gh}
\coeffentry{733}{\tfrac{329}{21600}}{aa cf dh eg}{bh dg eg fh}
\coeffentry{734}{\tfrac{421}{67200}}{aa cf dh eg}{bh dg eh fg}
\coeffentry{735}{\tfrac{137}{134400}}{aa cg de fg}{be bf fh hh}
\coeffentry{736}{\tfrac{5389}{1209600}}{aa cg de fg}{be df fh hh}
\coeffentry{737}{\tfrac{1}{9600}}{aa cg de fg}{bf bh eg hh}
\coeffentry{738}{-\tfrac{937}{403200}}{aa cg de fg}{bf bh eh hh}
\coeffentry{739}{\tfrac{1}{151200}}{aa cg de fg}{bf ce dh hh}
\coeffentry{740}{\tfrac{59}{13440}}{aa cg de fg}{bf dh ef hh}
\coeffentry{741}{\tfrac{1}{8400}}{aa cg de fg}{bf dh eg hh}
\coeffentry{742}{\tfrac{3037}{1209600}}{aa cg de fg}{bf ef eh hh}
\coeffentry{743}{-\tfrac{439}{172800}}{aa cg de fg}{bf ef fh hh}
\coeffentry{744}{\tfrac{931}{172800}}{aa cg de fg}{bf ef gh hh}
\coeffentry{745}{-\tfrac{1}{9600}}{aa cg de fg}{bf eg fh hh}
\coeffentry{746}{\tfrac{4}{945}}{aa cg de fg}{bf eg gh hh}
\coeffentry{747}{\tfrac{533}{302400}}{aa cg de fg}{bf eh fh hh}
\coeffentry{748}{-\tfrac{173}{57600}}{aa cg de fg}{bf eh gh hh}
\coeffentry{749}{-\tfrac{103}{302400}}{aa cg de fg}{bg ef fh hh}
\coeffentry{750}{\tfrac{409}{172800}}{aa cg de fg}{bh ef fh hh}
\coeffentry{751}{\tfrac{3343}{1209600}}{aa cg de fh}{bf bg bh eh}
\coeffentry{752}{-\tfrac{12521}{1209600}}{aa cg de fh}{bf bg dh eh}
\coeffentry{753}{-\tfrac{1483}{86400}}{aa cg de fh}{bf bg eg hh}
\coeffentry{754}{\tfrac{8089}{1209600}}{aa cg de fh}{bf bg eh hh}
\coeffentry{755}{-\tfrac{8717}{604800}}{aa cg de fh}{bf bh eg gh}
\coeffentry{756}{\tfrac{473}{604800}}{aa cg de fh}{bf eg gh hh}
\coeffentry{757}{-\tfrac{46159}{1209600}}{aa cg de fh}{bf eh gh gh}
\coeffentry{758}{\tfrac{1279}{151200}}{aa cg de fh}{bh ef fg gh}
\coeffentry{759}{\tfrac{3149}{403200}}{aa cg de fh}{bh eg fg fh}
\coeffentry{760}{-\tfrac{13009}{1209600}}{aa cg de fh}{bh eh fg fg}
\coeffentry{761}{\tfrac{1}{33600}}{aa cg df eg}{be bf bh hh}
\coeffentry{762}{\tfrac{61}{43200}}{aa cg df eg}{be bf dh hh}
\coeffentry{763}{-\tfrac{653}{604800}}{aa cg df eg}{be bf fh hh}
\coeffentry{764}{\tfrac{881}{1209600}}{aa cg df eg}{be bf gh hh}
\coeffentry{765}{-\tfrac{667}{67200}}{aa cg df eg}{be fg fh hh}
\coeffentry{766}{\tfrac{1213}{604800}}{aa cg df eg}{bf bh ef hh}
\coeffentry{767}{-\tfrac{161}{172800}}{aa cg df eg}{bf bh eh hh}
\coeffentry{768}{-\tfrac{39}{8960}}{aa cg df eg}{bf ef eh hh}
\coeffentry{769}{-\tfrac{463}{604800}}{aa cg df eg}{bf ef fh hh}
\coeffentry{770}{\tfrac{983}{403200}}{aa cg df eg}{bf ef gh hh}
\coeffentry{771}{\tfrac{503}{241920}}{aa cg df eg}{bf eh fg hh}
\coeffentry{772}{-\tfrac{23}{60480}}{aa cg df eg}{bf eh fh hh}
\coeffentry{773}{-\tfrac{1061}{302400}}{aa cg df eg}{bg ef fh hh}
\coeffentry{774}{\tfrac{667}{604800}}{aa cg df eg}{bg ef gh hh}
\coeffentry{775}{-\tfrac{11}{151200}}{aa cg df eg}{bg eh fg hh}
\coeffentry{776}{\tfrac{379}{172800}}{aa cg df eg}{bg eh fh hh}
\coeffentry{777}{-\tfrac{19}{16800}}{aa cg df eg}{bh ef gh hh}
\coeffentry{778}{\tfrac{41}{201600}}{aa cg df eg}{bh eh fh hh}
\coeffentry{779}{-\tfrac{137}{86400}}{aa cg df eh}{be bf bg hh}
\coeffentry{780}{\tfrac{37}{60480}}{aa cg df eh}{bf bg bh eh}
\coeffentry{781}{\tfrac{2077}{403200}}{aa cg df eh}{bf bg eg hh}
\coeffentry{782}{\tfrac{359}{57600}}{aa cg df eh}{bg bh bh ef}
\coeffentry{783}{\tfrac{619}{120960}}{aa cg df eh}{bg bh eh fh}
\coeffentry{784}{\tfrac{2033}{241920}}{aa cg df eh}{bg ef fg hh}
\coeffentry{785}{-\tfrac{31}{172800}}{aa cg df eh}{bg eg fg hh}
\coeffentry{786}{-\tfrac{77}{43200}}{aa cg df eh}{bg eg fh hh}
\coeffentry{787}{\tfrac{953}{302400}}{aa cg df eh}{bg eh fh gh}
\coeffentry{788}{\tfrac{719}{151200}}{aa cg df eh}{bh bh eg fg}
\coeffentry{789}{\tfrac{23}{15120}}{aa cg df eh}{bh eg fg fh}
\coeffentry{790}{-\tfrac{17629}{604800}}{aa cg df eh}{bh eg fg hh}
\coeffentry{791}{-\tfrac{3457}{1209600}}{aa cg df eh}{bh eh fg gh}
\coeffentry{792}{-\tfrac{19}{25200}}{aa cg df eh}{bh eh fh gh}
\coeffentry{793}{\tfrac{149}{75600}}{aa cg dg ef}{be df fh hh}
\coeffentry{794}{-\tfrac{47}{26880}}{aa cg dg ef}{be fg fh hh}
\coeffentry{795}{-\tfrac{79}{100800}}{aa cg dg ef}{bf bg eh hh}
\coeffentry{796}{-\tfrac{1453}{1209600}}{aa cg dg ef}{bf bh de hh}
\coeffentry{797}{\tfrac{103}{302400}}{aa cg dg ef}{bf bh eg hh}
\coeffentry{798}{\tfrac{47}{151200}}{aa cg dg ef}{bf ch de hh}
\coeffentry{799}{\tfrac{451}{403200}}{aa cg dg ef}{bf de fh hh}
\coeffentry{800}{-\tfrac{1843}{1209600}}{aa cg dg ef}{bf de gh hh}
\coeffentry{801}{-\tfrac{307}{604800}}{aa cg dg ef}{bf dh eg hh}
\coeffentry{802}{-\tfrac{281}{241920}}{aa cg dg ef}{bf dh eh hh}
\coeffentry{803}{\tfrac{7}{14400}}{aa cg dg ef}{bf eg gh hh}
\coeffentry{804}{\tfrac{1}{6048}}{aa cg dg ef}{bf eh gh hh}
\coeffentry{805}{\tfrac{41}{37800}}{aa cg dg ef}{bf fg fh hh}
\coeffentry{806}{-\tfrac{2641}{604800}}{aa cg dg ef}{bf fg gh hh}
\coeffentry{807}{\tfrac{673}{302400}}{aa cg dg ef}{bf fh gh hh}
\coeffentry{808}{-\tfrac{13}{9600}}{aa cg dg ef}{bg eh fg hh}
\coeffentry{809}{\tfrac{31}{8400}}{aa cg dg ef}{bg fg fh hh}
\coeffentry{810}{\tfrac{11}{60480}}{aa cg dg ef}{bg fg gh hh}
\coeffentry{811}{-\tfrac{379}{172800}}{aa cg dg ef}{bg fh gh hh}
\coeffentry{812}{\tfrac{97}{302400}}{aa cg dg ef}{bh df eh hh}
\coeffentry{813}{\tfrac{443}{1209600}}{aa cg dg ef}{bh eh fg hh}
\coeffentry{814}{-\tfrac{743}{302400}}{aa cg dg ef}{bh fg fh hh}
\coeffentry{815}{-\tfrac{47}{604800}}{aa cg dg ef}{bh fh gh hh}
\coeffentry{816}{-\tfrac{8231}{1209600}}{aa cg dh ef}{bf dg eg hh}
\coeffentry{817}{-\tfrac{233}{22400}}{aa cg dh ef}{bg df eg hh}
\coeffentry{818}{-\tfrac{13}{4032}}{aa cg dh ef}{bh bh df eg}
\coeffentry{819}{-\tfrac{4519}{403200}}{aa cg dh ef}{bh df eg gh}
\coeffentry{820}{-\tfrac{47}{172800}}{aa cg dh ef}{bh df eh gh}
\coeffentry{821}{-\tfrac{5}{5376}}{aa cg dh ef}{bh dg eh fg}
\coeffentry{822}{-\tfrac{277}{302400}}{aa cg dh ef}{bh dg eh fh}
\coeffentry{823}{\tfrac{233}{1209600}}{aa cg dh ef}{bh dh eh fg}
\coeffentry{824}{-\tfrac{4019}{1209600}}{aa ch de fg}{bf bg bh eh}
\coeffentry{825}{\tfrac{7151}{604800}}{aa ch de fg}{bf dh eg gh}
\coeffentry{826}{-\tfrac{1093}{604800}}{aa ch de fg}{bf eh gh gh}
\coeffentry{827}{-\tfrac{2099}{1209600}}{aa ch de fg}{bg cf dh eh}
\coeffentry{828}{-\tfrac{4051}{604800}}{aa ch df eg}{bf bh eg gh}
\coeffentry{829}{\tfrac{4433}{403200}}{aa ch df eg}{bg bh dh ef}
\coeffentry{830}{\tfrac{1037}{172800}}{aa ch df eg}{bg bh eh fh}
\coeffentry{831}{-\tfrac{1307}{604800}}{aa ch df eg}{bg ch dh ef}
\coeffentry{832}{\tfrac{211}{151200}}{aa ch df eg}{bg dh ef hh}
\coeffentry{833}{\tfrac{1859}{1209600}}{aa ch df eg}{bh cg dh ef}
\coeffentry{834}{-\tfrac{379}{75600}}{aa ch df eg}{bh eg fh gh}
\coeffentry{835}{\tfrac{589}{604800}}{aa ch dg ef}{be bf bg bh}
\coeffentry{836}{\tfrac{15907}{1209600}}{aa ch dg ef}{be fg fh gh}
\coeffentry{837}{-\tfrac{4763}{1209600}}{aa ch dg ef}{bf bg dh eh}
\coeffentry{838}{\tfrac{11}{5376}}{aa ch dg ef}{bf bh eg gh}
\coeffentry{839}{-\tfrac{33}{8960}}{aa ch dg ef}{bf dg eh gh}
\coeffentry{840}{-\tfrac{3679}{241920}}{aa ch dg ef}{bf dh eg gh}
\coeffentry{841}{-\tfrac{607}{1209600}}{aa ch dg ef}{bg bh cf de}
\coeffentry{842}{\tfrac{1093}{302400}}{aa ch dg ef}{bg bh fg fh}
\coeffentry{843}{-\tfrac{481}{1209600}}{aa ch dg ef}{bg cf de hh}
\coeffentry{844}{-\tfrac{1}{6400}}{aa ch dg ef}{bg dh eh fh}
\coeffentry{845}{-\tfrac{1037}{151200}}{aa ch dg ef}{bg eh fg hh}
\coeffentry{846}{\tfrac{139}{302400}}{aa ch dg ef}{bg eh fh gh}
\coeffentry{847}{\tfrac{701}{86400}}{aa ch dg ef}{bh dg eh fg}
\coeffentry{848}{-\tfrac{7}{14400}}{aa ch dg ef}{bh dh eh fg}
\coeffentry{849}{\tfrac{53}{50400}}{aa ch dg ef}{bh eh fh gh}
\coeffentry{850}{-\tfrac{1}{120960}}{aa de df dg}{bb bc cc dh}
\coeffentry{851}{-\tfrac{1}{120960}}{aa de df dg}{bc bc bc dh}
\coeffentry{852}{-\tfrac{1003}{201600}}{aa de fg gh}{bc cf ef hh}
\coeffentry{853}{\tfrac{869}{302400}}{aa de fg gh}{bc cf eg hh}
\coeffentry{854}{\tfrac{643}{604800}}{aa de fg gh}{bc cf eh hh}
\coeffentry{855}{\tfrac{185}{48384}}{aa de fg gh}{bf ce cf hh}
\coeffentry{856}{-\tfrac{181}{19200}}{aa de fg gh}{bf ce cg hh}
\coeffentry{857}{\tfrac{1193}{604800}}{aa de fg gh}{bf ce df hh}
\coeffentry{858}{\tfrac{117}{44800}}{aa de fg gh}{bf cf ef hh}
\coeffentry{859}{\tfrac{947}{1209600}}{aa de fg gh}{bf cf eg hh}
\coeffentry{860}{-\tfrac{29}{10800}}{aa de fg gh}{bf cf eh hh}
\coeffentry{861}{-\tfrac{1219}{302400}}{aa de fg gh}{bg cf ef hh}
\coeffentry{862}{\tfrac{377}{302400}}{aa de fg gh}{bg cf eg hh}
\coeffentry{863}{\tfrac{4859}{604800}}{aa de fg gh}{bg cf eh hh}
\coeffentry{864}{\tfrac{113}{75600}}{aa de fg hh}{bc cf eg gh}
\coeffentry{865}{-\tfrac{1511}{604800}}{aa de fg hh}{bg cf dh eg}
\coeffentry{866}{-\tfrac{4337}{604800}}{aa de fg hh}{bg cf eg fh}
\coeffentry{867}{-\tfrac{79}{1209600}}{aa de fh gh}{bc bg cf eg}
\coeffentry{868}{-\tfrac{143}{604800}}{aa de fh gh}{bc bg cf eh}
\coeffentry{869}{-\tfrac{1433}{604800}}{aa de fh gh}{bc cf cg eg}
\coeffentry{870}{\tfrac{529}{241920}}{aa de fh gh}{bc cf cg eh}
\coeffentry{871}{\tfrac{907}{201600}}{aa de fh gh}{bc cf dg eg}
\coeffentry{872}{-\tfrac{1913}{1209600}}{aa de fh gh}{bc cf eg gh}
\coeffentry{873}{-\tfrac{269}{120960}}{aa de fh gh}{bc cg ef fg}
\coeffentry{874}{\tfrac{1}{7560}}{aa de fh gh}{bf ce cg dg}
\coeffentry{875}{\tfrac{3331}{1209600}}{aa de fh gh}{bf ce cg gh}
\coeffentry{876}{\tfrac{4129}{1209600}}{aa de fh gh}{bg ce cf fg}
\coeffentry{877}{\tfrac{73}{16800}}{aa de fh gh}{bg ce df fg}
\coeffentry{878}{\tfrac{367}{60480}}{aa de fh gh}{bg cf cf eg}
\coeffentry{879}{-\tfrac{5377}{1209600}}{aa de fh gh}{bg cf cg eg}
\coeffentry{880}{-\tfrac{2377}{1209600}}{aa de fh gh}{bg cf cg eh}
\coeffentry{881}{-\tfrac{4813}{1209600}}{aa de fh gh}{bg cf ch eg}
\coeffentry{882}{\tfrac{67}{8064}}{aa de fh gh}{bg cf ch eh}
\coeffentry{883}{\tfrac{107}{80640}}{aa de fh gh}{bg cf dg ef}
\coeffentry{884}{\tfrac{9}{8960}}{aa de fh gh}{bg cf dg eg}
\coeffentry{885}{-\tfrac{1753}{1209600}}{aa de fh gh}{bg cf dh eg}
\coeffentry{886}{-\tfrac{247}{201600}}{aa de fh gh}{bg cf dh eh}
\coeffentry{887}{-\tfrac{139}{33600}}{aa de fh gh}{bg cf ef eg}
\coeffentry{888}{\tfrac{271}{100800}}{aa de fh gh}{bg cf eg fg}
\coeffentry{889}{\tfrac{2749}{1209600}}{aa de fh gh}{bg cf eh fg}
\coeffentry{890}{-\tfrac{17}{2400}}{aa de fh gh}{bg cf eh gh}
\coeffentry{891}{-\tfrac{1}{432}}{aa de fh gh}{bg cg df ef}
\coeffentry{892}{\tfrac{1457}{604800}}{aa de fh gh}{bg cg ef fg}
\coeffentry{893}{-\tfrac{43}{50400}}{aa de fh gh}{bg cg ef fh}
\coeffentry{894}{\tfrac{2369}{1209600}}{aa de fh gh}{bh cf cg eg}
\coeffentry{895}{-\tfrac{359}{100800}}{aa de fh gh}{bh cf cg eh}
\coeffentry{896}{\tfrac{3091}{1209600}}{aa de fh gh}{bh cf dg eg}
\coeffentry{897}{-\tfrac{47}{15120}}{aa de fh gh}{bh cf eg gh}
\coeffentry{898}{-\tfrac{5}{2304}}{aa de fh gh}{bh cg ef fg}
\coeffentry{899}{-\tfrac{23}{4725}}{aa df eg fg}{bc ce gh hh}
\coeffentry{900}{\tfrac{13}{86400}}{aa df eg fg}{bc cg eh hh}
\coeffentry{901}{\tfrac{37}{302400}}{aa df eg fg}{be ce gh hh}
\coeffentry{902}{\tfrac{131}{1209600}}{aa df eg fg}{be cg ch hh}
\coeffentry{903}{\tfrac{1}{18900}}{aa df eg fg}{bg ce ch hh}
\coeffentry{904}{-\tfrac{1}{5400}}{aa df eg fg}{bg ce dh hh}
\coeffentry{905}{\tfrac{11}{302400}}{aa df eg fg}{bg ce eh hh}
\coeffentry{906}{\tfrac{6241}{1209600}}{aa df eg fg}{bg ce fh hh}
\coeffentry{907}{-\tfrac{1657}{302400}}{aa df eg fg}{bg ce gh hh}
\coeffentry{908}{\tfrac{1}{30240}}{aa df eg fg}{bg cg eh hh}
\coeffentry{909}{-\tfrac{1}{67200}}{aa df eg fg}{bh ce gh hh}
\coeffentry{910}{\tfrac{19}{86400}}{aa df eg fg}{bh cg eh hh}
\coeffentry{911}{\tfrac{79}{43200}}{aa df eg gh}{bc bc ef hh}
\coeffentry{912}{\tfrac{803}{604800}}{aa df eg gh}{bc bf ce hh}
\coeffentry{913}{-\tfrac{3223}{1209600}}{aa df eg gh}{bc ce cf hh}
\coeffentry{914}{\tfrac{81}{44800}}{aa df eg gh}{bc cf ef hh}
\coeffentry{915}{-\tfrac{1}{60480}}{aa df eg gh}{bf ce cf hh}
\coeffentry{916}{-\tfrac{317}{172800}}{aa df eg gh}{bf ce dg hh}
\coeffentry{917}{\tfrac{127}{67200}}{aa df eg gh}{bf ce dh hh}
\coeffentry{918}{-\tfrac{19}{9600}}{aa df eg gh}{bf ce ef hh}
\coeffentry{919}{\tfrac{209}{43200}}{aa df eg gh}{bf ce fg hh}
\coeffentry{920}{-\tfrac{1}{10800}}{aa df eg gh}{bf cf ef hh}
\coeffentry{921}{-\tfrac{3247}{1209600}}{aa df eg gh}{bg ce cf hh}
\coeffentry{922}{\tfrac{647}{403200}}{aa df eg gh}{bg cf ef hh}
\coeffentry{923}{-\tfrac{11}{22400}}{aa df eg hh}{bc cg dh ef}
\coeffentry{924}{-\tfrac{533}{604800}}{aa df eg hh}{bc cg eh fh}
\coeffentry{925}{\tfrac{377}{57600}}{aa df eh gh}{bc bc cg ef}
\coeffentry{926}{\tfrac{53}{17280}}{aa df eh gh}{bc bg ce cf}
\coeffentry{927}{-\tfrac{41}{37800}}{aa df eh gh}{bc bg cf eg}
\coeffentry{928}{-\tfrac{1537}{1209600}}{aa df eh gh}{bc cf cg eg}
\coeffentry{929}{\tfrac{19}{30240}}{aa df eh gh}{bc cg ef fg}
\coeffentry{930}{-\tfrac{143}{604800}}{aa df eh gh}{bc cg eg fh}
\coeffentry{931}{-\tfrac{47}{604800}}{aa df eh gh}{bc cg eh fh}
\coeffentry{932}{-\tfrac{4939}{1209600}}{aa df eh gh}{bf ce cg dg}
\coeffentry{933}{-\tfrac{163}{43200}}{aa df eh gh}{bg ce cf dg}
\coeffentry{934}{\tfrac{17}{86400}}{aa df eh gh}{bg ce cf dh}
\coeffentry{935}{\tfrac{5129}{1209600}}{aa df eh gh}{bg ce cf fg}
\coeffentry{936}{\tfrac{1019}{241920}}{aa df eh gh}{bg ce df fg}
\coeffentry{937}{-\tfrac{527}{201600}}{aa df eh gh}{bg ce fg fg}
\coeffentry{938}{\tfrac{317}{134400}}{aa df eh gh}{bg ce fg fh}
\coeffentry{939}{-\tfrac{1}{48384}}{aa df eh gh}{bg cf cg ef}
\coeffentry{940}{-\tfrac{3631}{1209600}}{aa df eh gh}{bg cf eg fg}
\coeffentry{941}{-\tfrac{1937}{604800}}{aa df eh gh}{bg cg ef ef}
\coeffentry{942}{\tfrac{3601}{1209600}}{aa df eh gh}{bg cg ef fg}
\coeffentry{943}{\tfrac{1223}{1209600}}{aa df eh gh}{bg cg ef fh}
\coeffentry{944}{-\tfrac{17}{60480}}{aa df eh gh}{bg cg eg fg}
\coeffentry{945}{-\tfrac{803}{403200}}{aa df eh gh}{bg cg eg fh}
\coeffentry{946}{\tfrac{109}{1209600}}{aa df eh gh}{bg cg eh fg}
\coeffentry{947}{\tfrac{821}{201600}}{aa df eh gh}{bg cg eh fh}
\coeffentry{948}{\tfrac{1139}{604800}}{aa df eh gh}{bh cg ch ef}
\coeffentry{949}{-\tfrac{2}{4725}}{aa df eh gh}{bh cg ef fg}
\coeffentry{950}{-\tfrac{107}{302400}}{aa df eh gh}{bh cg eh fg}
\coeffentry{951}{-\tfrac{11}{60480}}{aa df eh gh}{bh cg eh fh}
\coeffentry{952}{-\tfrac{1}{14400}}{aa dg ef fg}{bc bc eh hh}
\coeffentry{953}{\tfrac{1}{8960}}{aa dg ef fg}{bc bc gh hh}
\coeffentry{954}{\tfrac{1}{25200}}{aa dg ef fg}{bc bh ce hh}
\coeffentry{955}{\tfrac{163}{1209600}}{aa dg ef fg}{bc bh cg hh}
\coeffentry{956}{-\tfrac{1}{25200}}{aa dg ef fg}{bc ce ch hh}
\coeffentry{957}{-\tfrac{1}{25200}}{aa dg ef fg}{bc ce dh hh}
\coeffentry{958}{\tfrac{6059}{1209600}}{aa dg ef fg}{bc ce gh hh}
\coeffentry{959}{\tfrac{29}{120960}}{aa dg ef fg}{bc cg ch hh}
\coeffentry{960}{-\tfrac{163}{604800}}{aa dg ef fg}{bc ch eh hh}
\coeffentry{961}{-\tfrac{41}{134400}}{aa dg ef fg}{bc ch gh hh}
\coeffentry{962}{-\tfrac{1}{8400}}{aa dg ef fg}{be cd ch hh}
\coeffentry{963}{\tfrac{23}{60480}}{aa dg ef fg}{bg cf gh hh}
\coeffentry{964}{\tfrac{139}{604800}}{aa dg ef fg}{bh ce ch hh}
\coeffentry{965}{-\tfrac{1}{40320}}{aa dg ef fg}{bh cg fh hh}
\coeffentry{966}{-\tfrac{19}{86400}}{aa dg ef fg}{bh ch gh hh}
\coeffentry{967}{\tfrac{13}{38400}}{aa dg ef gh}{bc bc fg hh}
\coeffentry{968}{-\tfrac{37}{57600}}{aa dg ef gh}{bc cf de hh}
\coeffentry{969}{-\tfrac{667}{172800}}{aa dg ef gh}{bc cf eg hh}
\coeffentry{970}{\tfrac{23}{5400}}{aa dg ef gh}{bf cd ce hh}
\coeffentry{971}{-\tfrac{269}{604800}}{aa dg ef gh}{bf ce df hh}
\coeffentry{972}{-\tfrac{1019}{1209600}}{aa dg ef gh}{bf ce fg hh}
\coeffentry{973}{-\tfrac{4087}{1209600}}{aa dg ef gh}{bf cf fg hh}
\coeffentry{974}{\tfrac{511}{172800}}{aa dg ef gh}{bf cf fh hh}
\coeffentry{975}{-\tfrac{737}{604800}}{aa dg ef gh}{bg cf eg hh}
\coeffentry{976}{\tfrac{37}{17280}}{aa dg ef gh}{bg cf fg hh}
\coeffentry{977}{\tfrac{73}{403200}}{aa dg eg fg}{bc bc fh hh}
\coeffentry{978}{\tfrac{1}{8960}}{aa dg eg fg}{bc bc gh hh}
\coeffentry{979}{\tfrac{53}{86400}}{aa dg eg fg}{bc bh cf hh}
\coeffentry{980}{-\tfrac{23}{40320}}{aa dg eg fg}{bc bh cg hh}
\coeffentry{981}{-\tfrac{13}{22400}}{aa dg eg fg}{bc cf ch hh}
\coeffentry{982}{\tfrac{1}{12600}}{aa dg eg fg}{bc cf eh hh}
\coeffentry{983}{-\tfrac{1}{40320}}{aa dg eg fg}{bc cf fh hh}
\coeffentry{984}{-\tfrac{79}{302400}}{aa dg eg fg}{bc cf gh hh}
\coeffentry{985}{\tfrac{47}{86400}}{aa dg eg fg}{bc cg ch hh}
\coeffentry{986}{\tfrac{53}{120960}}{aa dg eg fg}{bc cg fh hh}
\coeffentry{987}{-\tfrac{11}{60480}}{aa dg eg fg}{bc cg gh hh}
\coeffentry{988}{\tfrac{1}{43200}}{aa dg eg fg}{bc ch fh hh}
\coeffentry{989}{\tfrac{1}{37800}}{aa dg eg fg}{bc ch gh hh}
\coeffentry{990}{\tfrac{11}{201600}}{aa dg eg fg}{bf ce ch hh}
\coeffentry{991}{-\tfrac{1}{151200}}{aa dg eg fg}{bf ce dh hh}
\coeffentry{992}{-\tfrac{1}{120960}}{aa dg eg fg}{bf ce fh hh}
\coeffentry{993}{-\tfrac{89}{151200}}{aa dg eg fg}{bf cf ch hh}
\coeffentry{994}{\tfrac{13}{22400}}{aa dg eg fg}{bf cf fh hh}
\coeffentry{995}{\tfrac{13}{60480}}{aa dg eg fg}{bf cf gh hh}
\coeffentry{996}{\tfrac{19}{604800}}{aa dg eg fg}{bg cg ch hh}
\coeffentry{997}{\tfrac{101}{604800}}{aa dg eg fg}{bg cg gh hh}
\coeffentry{998}{\tfrac{1}{120960}}{aa dg eg fg}{bh cf ch hh}
\coeffentry{999}{\tfrac{1}{43200}}{aa dg eg fg}{bh cf eh hh}
\coeffentry{1000}{-\tfrac{19}{604800}}{aa dg eg fg}{bh cf fh hh}
\coeffentry{1001}{\tfrac{1}{604800}}{aa dg eg fg}{bh cg ch hh}
\coeffentry{1002}{-\tfrac{19}{604800}}{aa dg eg fg}{bh cg gh hh}
\coeffentry{1003}{-\tfrac{1}{604800}}{aa dg eg fg}{bh ch gh hh}
\coeffentry{1004}{\tfrac{293}{604800}}{aa dg eg fh}{bc bc fg hh}
\coeffentry{1005}{\tfrac{157}{1209600}}{aa dg eg fh}{bc bg cf hh}
\coeffentry{1006}{-\tfrac{377}{604800}}{aa dg eg fh}{bc cf cg hh}
\coeffentry{1007}{-\tfrac{359}{604800}}{aa dg eg fh}{bc cf ef hh}
\coeffentry{1008}{-\tfrac{809}{604800}}{aa dg eg fh}{bc cf eh hh}
\coeffentry{1009}{\tfrac{323}{201600}}{aa dg eg fh}{bc cf fg hh}
\coeffentry{1010}{\tfrac{761}{604800}}{aa dg eg fh}{bc cf gh hh}
\coeffentry{1011}{-\tfrac{1}{1440}}{aa dg eg fh}{bc cg fg hh}
\coeffentry{1012}{-\tfrac{11}{60480}}{aa dg eg fh}{bc ch fg hh}
\coeffentry{1013}{\tfrac{899}{241920}}{aa dg eg fh}{bf ce cf hh}
\coeffentry{1014}{-\tfrac{47}{1209600}}{aa dg eg fh}{bf ce ch hh}
\coeffentry{1015}{\tfrac{409}{201600}}{aa dg eg fh}{bf ce df hh}
\coeffentry{1016}{-\tfrac{71}{50400}}{aa dg eg fh}{bf ce fg hh}
\coeffentry{1017}{-\tfrac{97}{18900}}{aa dg eg fh}{bf cf cg hh}
\coeffentry{1018}{-\tfrac{577}{241920}}{aa dg eg fh}{bf cf ef hh}
\coeffentry{1019}{\tfrac{1289}{604800}}{aa dg eg fh}{bf cf eh hh}
\coeffentry{1020}{\tfrac{607}{403200}}{aa dg eg fh}{bf cf fg hh}
\coeffentry{1021}{-\tfrac{1241}{604800}}{aa dg eg fh}{bf cf gh hh}
\coeffentry{1022}{\tfrac{431}{302400}}{aa dg eg fh}{bg cf cg hh}
\coeffentry{1023}{\tfrac{6649}{1209600}}{aa dg eg fh}{bg cf ef hh}
\coeffentry{1024}{-\tfrac{61}{50400}}{aa dg eg fh}{bg cf fg hh}
\coeffentry{1025}{\tfrac{11}{60480}}{aa dg eg fh}{bg cf gh hh}
\coeffentry{1026}{\tfrac{19}{33600}}{aa dg eg fh}{bg cg fg hh}
\coeffentry{1027}{-\tfrac{59}{57600}}{aa dg eg fh}{bh cf cg hh}
\coeffentry{1028}{-\tfrac{23}{1209600}}{aa dg eg fh}{bh cf ef hh}
\coeffentry{1029}{\tfrac{89}{100800}}{aa dg eg fh}{bh cf eh hh}
\coeffentry{1030}{\tfrac{983}{1209600}}{aa dg eg fh}{bh cf fg hh}
\coeffentry{1031}{-\tfrac{41}{302400}}{aa dg eg fh}{bh cf gh hh}
\coeffentry{1032}{\tfrac{13}{9600}}{aa dg eg fh}{bh cg fg hh}
\coeffentry{1033}{-\tfrac{443}{1209600}}{aa dg eg fh}{bh ch fg hh}
\coeffentry{1034}{\tfrac{6599}{2419200}}{aa dg eh fg}{bc bc ef hh}
\coeffentry{1035}{\tfrac{439}{1209600}}{aa dg eh fg}{bc bf ce hh}
\coeffentry{1036}{\tfrac{223}{151200}}{aa dg eh fg}{bc ce cf hh}
\coeffentry{1037}{-\tfrac{773}{201600}}{aa dg eh fg}{bc cf ef hh}
\coeffentry{1038}{-\tfrac{881}{604800}}{aa dg eh fg}{bc cf eg hh}
\coeffentry{1039}{-\tfrac{317}{1209600}}{aa dg eh fg}{bc cg ef hh}
\coeffentry{1040}{\tfrac{101}{403200}}{aa dg eh fg}{bf ce cf hh}
\coeffentry{1041}{\tfrac{1657}{604800}}{aa dg eh fg}{bf ce dh hh}
\coeffentry{1042}{-\tfrac{317}{134400}}{aa dg eh fg}{bf ce gh hh}
\coeffentry{1043}{\tfrac{997}{1209600}}{aa dg eh fg}{bf cf ef hh}
\coeffentry{1044}{-\tfrac{1}{13440}}{aa dg eh fg}{bh ce cf hh}
\coeffentry{1045}{\tfrac{407}{1209600}}{aa dg eh fg}{bh cf ef hh}
\coeffentry{1046}{\tfrac{1}{1260}}{aa dg eh fg}{bh cf eg hh}
\coeffentry{1047}{-\tfrac{313}{604800}}{aa dg eh fg}{bh ch ef hh}
\coeffentry{1048}{-\tfrac{409}{1209600}}{aa dg eh fh}{bc bg cf eg}
\coeffentry{1049}{-\tfrac{113}{151200}}{aa dg eh fh}{bc cg eg fg}
\coeffentry{1050}{\tfrac{857}{1209600}}{aa dg eh fh}{bf ce cg dg}
\coeffentry{1051}{\tfrac{283}{172800}}{aa dg eh fh}{bg ce cf dh}
\coeffentry{1052}{\tfrac{181}{403200}}{aa dg eh fh}{bg ce cf gh}
\coeffentry{1053}{-\tfrac{103}{21600}}{aa dg eh fh}{bg ce df fg}
\coeffentry{1054}{\tfrac{1037}{1209600}}{aa dg eh fh}{bg cf cg eg}
\coeffentry{1055}{-\tfrac{67}{134400}}{aa dg eh fh}{bg cf eg fg}
\coeffentry{1056}{-\tfrac{19}{50400}}{aa dg eh fh}{bg cg eg fg}
\coeffentry{1057}{\tfrac{13}{1209600}}{aa dg eh fh}{bg cg eh fg}
\coeffentry{1058}{-\tfrac{1}{7200}}{aa dg eh fh}{bh cg eh fg}
\coeffentry{1059}{-\tfrac{11}{60480}}{aa dg eh fh}{bh ch eh fg}
\coeffentry{1060}{-\tfrac{1403}{483840}}{aa dh ef gh}{bc bc eg fg}
\coeffentry{1061}{\tfrac{1}{6300}}{aa dh ef gh}{bc bg ce cf}
\coeffentry{1062}{-\tfrac{167}{302400}}{aa dh ef gh}{bc bg cf eg}
\coeffentry{1063}{\tfrac{983}{1209600}}{aa dh ef gh}{bc cf cg eg}
\coeffentry{1064}{\tfrac{1417}{604800}}{aa dh ef gh}{bc cf dg eg}
\coeffentry{1065}{\tfrac{881}{403200}}{aa dh ef gh}{bc cf eg gh}
\coeffentry{1066}{\tfrac{19}{151200}}{aa dh ef gh}{bc cg df eg}
\coeffentry{1067}{-\tfrac{233}{604800}}{aa dh ef gh}{bc cg df eh}
\coeffentry{1068}{\tfrac{439}{604800}}{aa dh ef gh}{bc cg ef fg}
\coeffentry{1069}{\tfrac{479}{1209600}}{aa dh ef gh}{bc cg eg fg}
\coeffentry{1070}{\tfrac{79}{302400}}{aa dh ef gh}{bc cg eh fg}
\coeffentry{1071}{\tfrac{71}{604800}}{aa dh ef gh}{bc cg eh fh}
\coeffentry{1072}{-\tfrac{389}{1209600}}{aa dh ef gh}{bf bg ce cg}
\coeffentry{1073}{-\tfrac{881}{403200}}{aa dh ef gh}{bf ce cg gh}
\coeffentry{1074}{-\tfrac{1187}{1209600}}{aa dh ef gh}{bg cd ce cf}
\coeffentry{1075}{\tfrac{47}{1209600}}{aa dh ef gh}{bg ce cf fg}
\coeffentry{1076}{-\tfrac{5}{48384}}{aa dh ef gh}{bg ce df fg}
\coeffentry{1077}{-\tfrac{43}{151200}}{aa dh ef gh}{bg cf cg de}
\coeffentry{1078}{\tfrac{19}{67200}}{aa dh ef gh}{bg cf cg eg}
\coeffentry{1079}{-\tfrac{37}{403200}}{aa dh ef gh}{bg cf dg eg}
\coeffentry{1080}{-\tfrac{127}{100800}}{aa dh ef gh}{bg cf eg fg}
\coeffentry{1081}{\tfrac{247}{201600}}{aa dh ef gh}{bg cf eh gh}
\coeffentry{1082}{-\tfrac{109}{1209600}}{aa dh ef gh}{bg cg df eg}
\coeffentry{1083}{-\tfrac{1}{60480}}{aa dh ef gh}{bg cg eg fg}
\coeffentry{1084}{-\tfrac{37}{241920}}{aa dh ef gh}{bg cg eh fg}
\coeffentry{1085}{-\tfrac{37}{67200}}{aa dh ef gh}{bg cg eh fh}
\coeffentry{1086}{\tfrac{1}{12600}}{aa dh ef gh}{bh cf eg gh}
\coeffentry{1087}{\tfrac{11}{60480}}{aa dh ef gh}{bh cg eh fh}
\coeffentry{1088}{-\tfrac{137}{172800}}{aa dh ef gh}{bh ch eg fg}
\coeffentry{1089}{-\tfrac{5}{24192}}{aa dh eg fg}{bg cg gh hh}
\coeffentry{1090}{\tfrac{443}{1209600}}{aa dh eg fg}{bh cg gh hh}
\coeffentry{1091}{\tfrac{19}{50400}}{aa dh eg fg}{bh ch gh hh}
\coeffentry{1092}{\tfrac{311}{1209600}}{aa dh eg fh}{bc bg cf eg}
\coeffentry{1093}{-\tfrac{1}{4800}}{aa dh eg fh}{bc bg cf eh}
\coeffentry{1094}{-\tfrac{221}{403200}}{aa dh eg fh}{bc cf cg eg}
\coeffentry{1095}{-\tfrac{29}{302400}}{aa dh eg fh}{bc cf cg eh}
\coeffentry{1096}{-\tfrac{403}{604800}}{aa dh eg fh}{bc cf eg gh}
\coeffentry{1097}{-\tfrac{1}{1890}}{aa dh eg fh}{bg ce fg fh}
\coeffentry{1098}{\tfrac{131}{302400}}{aa dh eg fh}{bh cf cg eh}
\coeffentry{1099}{-\tfrac{19}{151200}}{aa dh eh fg}{bc ce cf cg}
\coeffentry{1100}{\tfrac{173}{1209600}}{aa dh eh fg}{bg cf dg ef}
\coeffentry{1101}{-\tfrac{83}{604800}}{aa ef fg gh}{bc bc dg hh}
\coeffentry{1102}{-\tfrac{31}{201600}}{aa ef fg gh}{bd cd cg hh}
\coeffentry{1103}{\tfrac{11}{302400}}{aa ef fg gh}{bd cd dg hh}
\coeffentry{1104}{\tfrac{1}{6300}}{aa ef fg gh}{be cd dg hh}
\coeffentry{1105}{-\tfrac{31}{604800}}{aa ef fg gh}{bg cd de hh}
\coeffentry{1106}{\tfrac{29}{201600}}{aa ef fg gh}{bg cd df hh}
\coeffentry{1107}{-\tfrac{13}{80640}}{aa ef gh gh}{bc bc df eh}
\coeffentry{1108}{-\tfrac{1}{4200}}{aa ef gh hh}{bd cd cg fg}
\coeffentry{1109}{\tfrac{19}{302400}}{aa eg fg fh}{bc bc de hh}
\coeffentry{1110}{\tfrac{1}{10368}}{aa eg fg fh}{bc bd cd hh}
\coeffentry{1111}{-\tfrac{1}{100800}}{aa eg fg fh}{bd cd cd hh}
\coeffentry{1112}{-\tfrac{31}{151200}}{aa eg fg fh}{bd cd ce hh}
\coeffentry{1113}{-\tfrac{1}{37800}}{aa eg fg fh}{bd cd de hh}
\coeffentry{1114}{-\tfrac{17}{241920}}{aa eg fg gh}{bc bc df hh}
\coeffentry{1115}{\tfrac{127}{302400}}{aa eg fg gh}{bc bc dg hh}
\coeffentry{1116}{-\tfrac{13}{604800}}{aa eg fg gh}{bc bd cd hh}
\coeffentry{1117}{\tfrac{1}{67200}}{aa eg fg gh}{bd cd cd hh}
\coeffentry{1118}{-\tfrac{19}{604800}}{aa eg fg gh}{bd cd cf hh}
\coeffentry{1119}{\tfrac{7}{86400}}{aa eg fg gh}{bd cd cg hh}
\coeffentry{1120}{\tfrac{1}{25200}}{aa eg fg gh}{bd cd df hh}
\coeffentry{1121}{-\tfrac{11}{201600}}{aa eg fg gh}{bd cd dg hh}
\coeffentry{1122}{-\tfrac{37}{604800}}{aa eg fg gh}{bf cd df hh}
\coeffentry{1123}{-\tfrac{1}{33600}}{aa eg fg gh}{bf cf de hh}
\coeffentry{1124}{-\tfrac{1}{302400}}{aa eg fg gh}{bf cf df hh}
\coeffentry{1125}{\tfrac{1}{37800}}{aa eg fg hh}{bc bd cd fh}
\coeffentry{1126}{-\tfrac{1}{18900}}{aa eg fg hh}{bc bd cd gh}
\coeffentry{1127}{-\tfrac{19}{115200}}{aa eg fh gh}{bc bc df dh}
\coeffentry{1128}{-\tfrac{89}{1209600}}{aa eg fh gh}{bc bc dh fg}
\coeffentry{1129}{-\tfrac{1}{18900}}{aa eg fh gh}{bc bd cd fh}
\coeffentry{1130}{\tfrac{229}{403200}}{aa eg fh gh}{bc cd df dh}
\coeffentry{1131}{-\tfrac{227}{1209600}}{aa eg fh gh}{bd bh cd cf}
\coeffentry{1132}{-\tfrac{13}{60480}}{aa eg fh gh}{bd cd cf dh}
\coeffentry{1133}{-\tfrac{59}{302400}}{aa eg fh gh}{bd cd ch df}
\coeffentry{1134}{\tfrac{1}{6300}}{aa eg fh gh}{bd cd ch fg}
\coeffentry{1135}{-\tfrac{1}{3600}}{aa eg fh gh}{be cd df fh}
\coeffentry{1136}{\tfrac{1}{8400}}{aa eg fh gh}{bf cd dh ef}
\coeffentry{1137}{\tfrac{1}{6300}}{aa eg fh gh}{bh cf df ef}
\coeffentry{1138}{\tfrac{1}{37800}}{aa eh fg gh}{bc bd cd ef}
\coeffentry{1139}{\tfrac{1}{12600}}{aa eh fg gh}{bd bf cd ce}
\coeffentry{1140}{\tfrac{17}{120960}}{aa eh fg gh}{bd cd de df}
\coeffentry{1141}{-\tfrac{1}{14400}}{aa eh fh gh}{bc bc df dg}
\coeffentry{1142}{-\tfrac{73}{302400}}{aa eh fh gh}{bc bc dg dh}
\coeffentry{1143}{\tfrac{1}{181440}}{aa eh fh gh}{bc bd cd fg}
\coeffentry{1144}{-\tfrac{11}{151200}}{aa eh fh gh}{bc cd dg dh}
\coeffentry{1145}{-\tfrac{1}{25200}}{aa eh fh gh}{bd bg cd cf}
\coeffentry{1146}{-\tfrac{1}{43200}}{aa eh fh gh}{bd cd cd fg}
\coeffentry{1147}{\tfrac{1}{151200}}{aa eh fh gh}{bd cd cg ef}
\coeffentry{1148}{\tfrac{1}{21600}}{aa eh fh gh}{bd cd ch dg}
\coeffentry{1149}{\tfrac{1}{604800}}{aa eh fh gh}{bd cd df dg}
\coeffentry{1150}{\tfrac{1}{43200}}{aa eh fh gh}{bd cd dg ef}
\coeffentry{1151}{-\tfrac{13}{241920}}{ab ab ab dd}{cd ef eg gh}
\coeffentry{1152}{-\tfrac{1}{604800}}{ab ab bc cc}{de df dg dh}
\coeffentry{1153}{\tfrac{1}{25200}}{ab ab bc cc}{de dg ef gh}
\coeffentry{1154}{\tfrac{1}{60480}}{ab ab bc de}{cc df dg dh}
\coeffentry{1155}{-\tfrac{13}{1209600}}{ab ab cc de}{bc df dg dh}
\coeffentry{1156}{\tfrac{17}{302400}}{ab ab cd cd}{eg fg gh hh}
\coeffentry{1157}{-\tfrac{11}{302400}}{ab ab cd cd}{eh fh gh gh}
\coeffentry{1158}{-\tfrac{1}{9450}}{ab ab cd ce}{ab cf cg gh}
\coeffentry{1159}{-\tfrac{1}{3780}}{ab ab cd ce}{ab cg ch ef}
\coeffentry{1160}{-\tfrac{1}{12600}}{ab ab cd cf}{ab de fg fh}
\coeffentry{1161}{\tfrac{23}{86400}}{ab ab cd fg}{de ee gh hh}
\coeffentry{1162}{-\tfrac{191}{1209600}}{ab ab cd gh}{de ee fg fh}
\coeffentry{1163}{\tfrac{167}{1209600}}{ab ab cd gh}{de ee fg hh}
\coeffentry{1164}{\tfrac{3}{44800}}{ab ab cd gh}{de ee fh fh}
\coeffentry{1165}{-\tfrac{17}{37800}}{ab ab dd ef}{ab cd eg gh}
\coeffentry{1166}{\tfrac{1}{5400}}{ab ab dd ef}{ab cd eh fg}
\coeffentry{1167}{\tfrac{1}{1890}}{ab ab dd ef}{bb cd eg gh}
\coeffentry{1168}{\tfrac{53}{2419200}}{ab ab dd fg}{cd eg eh hh}
\coeffentry{1169}{\tfrac{11}{134400}}{ab ab dd fg}{cd eg fh hh}
\coeffentry{1170}{-\tfrac{11}{268800}}{ab ab dd gh}{cd ef eh fg}
\coeffentry{1171}{-\tfrac{113}{1209600}}{ab ab de fg}{cf eg gh hh}
\coeffentry{1172}{\tfrac{1}{43200}}{ab ab de gh}{cd ce fg fh}
\coeffentry{1173}{-\tfrac{1}{9450}}{ab ab de gh}{cd ce fh fh}
\coeffentry{1174}{\tfrac{1}{75600}}{ab ab de gh}{cd ee fg hh}
\coeffentry{1175}{\tfrac{17}{151200}}{ab ab de gh}{cd ee fh fh}
\coeffentry{1176}{-\tfrac{1}{11200}}{ab ab de gh}{ce ce fh fh}
\coeffentry{1177}{\tfrac{17}{172800}}{ab ab df eg}{cg ch ef hh}
\coeffentry{1178}{-\tfrac{17}{172800}}{ab ab df eg}{cg ef gh hh}
\coeffentry{1179}{\tfrac{47}{201600}}{ab ab df eg}{cg eh fg hh}
\coeffentry{1180}{\tfrac{1}{80640}}{ab ab df eg}{cg eh fh hh}
\coeffentry{1181}{\tfrac{1}{302400}}{ab ab dg ef}{cf cg ch hh}
\coeffentry{1182}{-\tfrac{1}{20160}}{ab ab dg ef}{cf cg eh hh}
\coeffentry{1183}{\tfrac{1}{21600}}{ab ab dg ef}{cf cg gh hh}
\coeffentry{1184}{-\tfrac{13}{151200}}{ab ab dg ef}{cf eg gh hh}
\coeffentry{1185}{-\tfrac{1}{302400}}{ab ab dg ef}{cg ch fh hh}
\coeffentry{1186}{\tfrac{1}{37800}}{ab ab dg ef}{cg eh fg hh}
\coeffentry{1187}{\tfrac{1}{43200}}{ab ab dg ef}{cg fg fh hh}
\coeffentry{1188}{-\tfrac{1}{18900}}{ab ab dg ef}{cg fg gh hh}
\coeffentry{1189}{\tfrac{1}{6720}}{ab ab dg ef}{cg fh gh hh}
\coeffentry{1190}{-\tfrac{17}{172800}}{ab ab dg ef}{ch fg gh hh}
\coeffentry{1191}{\tfrac{43}{604800}}{ab ab dg ef}{ch fh gh hh}
\coeffentry{1192}{-\tfrac{17}{172800}}{ab ab ef gh}{cd ch dg fh}
\coeffentry{1193}{\tfrac{103}{1209600}}{ab ab ef gh}{cd dg dh fh}
\coeffentry{1194}{-\tfrac{1}{12600}}{ab ab ef gh}{cd dg eh fh}
\coeffentry{1195}{\tfrac{1}{10800}}{ab ab ef gh}{cd dg fg hh}
\coeffentry{1196}{\tfrac{1}{5040}}{ab ab ef gh}{cd dg fh hh}
\coeffentry{1197}{-\tfrac{1}{12600}}{ab ab ef gh}{cd dh fg gh}
\coeffentry{1198}{\tfrac{1}{60480}}{ab ab ef gh}{cg df dg hh}
\coeffentry{1199}{-\tfrac{199}{1209600}}{ab ab ef gh}{cg df dh dh}
\coeffentry{1200}{\tfrac{1}{75600}}{ab ab ef gh}{cg df dh eh}
\coeffentry{1201}{-\tfrac{1}{75600}}{ab ab ef gh}{cg df eg hh}
\coeffentry{1202}{\tfrac{151}{1209600}}{ab ab ef gh}{cg dg fg hh}
\coeffentry{1203}{-\tfrac{229}{1209600}}{ab ab ef gh}{cg dg fh hh}
\coeffentry{1204}{\tfrac{31}{151200}}{ab ab ef gh}{cg dh dh fh}
\coeffentry{1205}{\tfrac{263}{1209600}}{ab ab ef gh}{ch df dg gh}
\coeffentry{1206}{\tfrac{1}{75600}}{ab ab ef gh}{ch df eg gh}
\coeffentry{1207}{-\tfrac{37}{403200}}{ab ab ef gh}{ch dg dg fh}
\coeffentry{1208}{-\tfrac{1}{8400}}{ab ab ef gh}{ch dg dh fh}
\coeffentry{1209}{-\tfrac{1}{15120}}{ab ab ef gh}{ch dg eh fg}
\coeffentry{1210}{-\tfrac{1}{75600}}{ab ab ef gh}{ch dg eh fh}
\coeffentry{1211}{-\tfrac{1}{75600}}{ab ab ef gh}{ch dg fg fh}
\coeffentry{1212}{\tfrac{1}{30240}}{ab ab ef gh}{ch dh eg fg}
\coeffentry{1213}{\tfrac{29}{1209600}}{ab ab ef gh}{ch dh fg gh}
\coeffentry{1214}{\tfrac{1}{50400}}{ab ab eg fg}{cd ch df hh}
\coeffentry{1215}{-\tfrac{1}{33600}}{ab ab eg fg}{cd df dh hh}
\coeffentry{1216}{\tfrac{1}{12600}}{ab ab eg fg}{cd df eh hh}
\coeffentry{1217}{\tfrac{1}{16800}}{ab ab eg fg}{cd dg dh hh}
\coeffentry{1218}{-\tfrac{1}{12600}}{ab ab eg fg}{cd dg fh hh}
\coeffentry{1219}{\tfrac{1}{50400}}{ab ab eg fg}{cd dg gh hh}
\coeffentry{1220}{-\tfrac{1}{16800}}{ab ab eg fg}{cd dh gh hh}
\coeffentry{1221}{-\tfrac{1}{50400}}{ab ab eg fg}{cf df dh hh}
\coeffentry{1222}{-\tfrac{1}{20160}}{ab ab eg fg}{cf df fh hh}
\coeffentry{1223}{\tfrac{1}{14400}}{ab ab eg fg}{cf df gh hh}
\coeffentry{1224}{-\tfrac{179}{1209600}}{ab ab eg fg}{ch df dh hh}
\coeffentry{1225}{\tfrac{179}{1209600}}{ab ab eg fg}{ch df fh hh}
\coeffentry{1226}{\tfrac{179}{1209600}}{ab ab eg fg}{ch dg dh hh}
\coeffentry{1227}{-\tfrac{19}{241920}}{ab ab eg fg}{ch dg gh hh}
\coeffentry{1228}{\tfrac{11}{75600}}{ab ab eg fg}{ch dh fh hh}
\coeffentry{1229}{-\tfrac{179}{1209600}}{ab ab eg fg}{ch dh gh hh}
\coeffentry{1230}{-\tfrac{1}{20160}}{ab ab eg fh}{cd cg df hh}
\coeffentry{1231}{\tfrac{13}{120960}}{ab ab eg fh}{cd df dg hh}
\coeffentry{1232}{\tfrac{1}{5400}}{ab ab eg fh}{cd dg fg hh}
\coeffentry{1233}{-\tfrac{1}{21600}}{ab ab eg fh}{cd dh fg gh}
\coeffentry{1234}{-\tfrac{1}{7200}}{ab ab eg fh}{cg df dg hh}
\coeffentry{1235}{-\tfrac{1}{37800}}{ab ab eg fh}{cg df dh eh}
\coeffentry{1236}{\tfrac{1}{37800}}{ab ab eg fh}{cg df fg hh}
\coeffentry{1237}{\tfrac{11}{151200}}{ab ab eg fh}{cg df gh hh}
\coeffentry{1238}{\tfrac{1}{12600}}{ab ab eg fh}{ch df dg eh}
\coeffentry{1239}{\tfrac{11}{151200}}{ab ab eg fh}{ch dg fh gh}
\coeffentry{1240}{\tfrac{1}{9450}}{ab ab eg fh}{ch dh fg fg}
\coeffentry{1241}{-\tfrac{43}{604800}}{ab ab eh fg}{cd df dg dh}
\coeffentry{1242}{\tfrac{1}{75600}}{ab ab eh fg}{cd dg dh ef}
\coeffentry{1243}{-\tfrac{11}{151200}}{ab ab eh fg}{cd dh fh gh}
\coeffentry{1244}{\tfrac{13}{151200}}{ab ab eh fg}{cg dh dh fh}
\coeffentry{1245}{-\tfrac{1}{12600}}{ab ab fg gh}{ab cd ce hh}
\coeffentry{1246}{\tfrac{1}{403200}}{ab ab fg gh}{cd ce de hh}
\coeffentry{1247}{\tfrac{1}{403200}}{ab ab fg gh}{ce de de hh}
\coeffentry{1248}{-\tfrac{1}{12600}}{ab ab fg gh}{ce de df hh}
\coeffentry{1249}{\tfrac{1}{14400}}{ab ab fg gh}{ce de dg hh}
\coeffentry{1250}{-\tfrac{1}{403200}}{ab ab fg gh}{ce de ef hh}
\coeffentry{1251}{-\tfrac{1}{172800}}{ab ab fg gh}{cf df ef hh}
\coeffentry{1252}{-\tfrac{1}{12600}}{ab ab fg hh}{ab cd ce gh}
\coeffentry{1253}{\tfrac{1}{12600}}{ab ab fh gh}{ab cd ce fg}
\coeffentry{1254}{\tfrac{1}{12600}}{ab ab fh gh}{cd de eg eh}
\coeffentry{1255}{-\tfrac{1}{50400}}{ab ab fh gh}{cd de eh eh}
\coeffentry{1256}{-\tfrac{1}{14400}}{ab ab fh gh}{ce de dh eg}
\coeffentry{1257}{\tfrac{1}{403200}}{ab ab fh gh}{ce de ef eg}
\coeffentry{1258}{-\tfrac{1}{60480}}{ab ac ad fg}{ef eg eh hh}
\coeffentry{1259}{\tfrac{1}{60480}}{ab ac ad fg}{ef eg gh hh}
\coeffentry{1260}{-\tfrac{23}{302400}}{ab ac ad fg}{ef gh gh hh}
\coeffentry{1261}{\tfrac{1}{100800}}{ab ac ad fg}{eg eg eh hh}
\coeffentry{1262}{\tfrac{1}{75600}}{ab ac ad fg}{eg eg fh hh}
\coeffentry{1263}{-\tfrac{1}{21600}}{ab ac ad fg}{eg eg gh hh}
\coeffentry{1264}{-\tfrac{1}{60480}}{ab ac ad fg}{eg eh eh hh}
\coeffentry{1265}{-\tfrac{1}{16800}}{ab ac ad fg}{eg eh fg hh}
\coeffentry{1266}{\tfrac{1}{25200}}{ab ac ad fg}{eg eh gh hh}
\coeffentry{1267}{\tfrac{1}{16800}}{ab ac ad fg}{eg fg fh hh}
\coeffentry{1268}{\tfrac{23}{302400}}{ab ac ad fg}{eg fh gh hh}
\coeffentry{1269}{-\tfrac{1}{16800}}{ab ac bb cc}{de dg ef gh}
\coeffentry{1270}{\tfrac{1}{75600}}{ab ac bc fg}{de dh eg hh}
\coeffentry{1271}{-\tfrac{1}{75600}}{ab ac bc fg}{de eg eh hh}
\coeffentry{1272}{-\tfrac{1}{75600}}{ab ac bc fg}{dh eg eh hh}
\coeffentry{1273}{\tfrac{1}{25200}}{ab ac cd fg}{ef eg eh hh}
\coeffentry{1274}{-\tfrac{1}{10080}}{ab ac cd fg}{ef eg gh hh}
\coeffentry{1275}{\tfrac{1}{50400}}{ab ac cd fg}{ef gh gh hh}
\coeffentry{1276}{-\tfrac{1}{7560}}{ab ac cd fg}{eg eg eh hh}
\coeffentry{1277}{\tfrac{11}{151200}}{ab ac cd fg}{eg eg gh hh}
\coeffentry{1278}{-\tfrac{23}{151200}}{ab ac cd fg}{eg eh eh hh}
\coeffentry{1279}{-\tfrac{1}{50400}}{ab ac cd fg}{eg eh fg hh}
\coeffentry{1280}{\tfrac{1}{6300}}{ab ac cd fg}{eg eh fh hh}
\coeffentry{1281}{\tfrac{1}{30240}}{ab ac cd fg}{eg eh gh hh}
\coeffentry{1282}{\tfrac{19}{151200}}{ab ac cd fg}{eg eh hh hh}
\coeffentry{1283}{\tfrac{1}{12600}}{ab ac cd fg}{eg fg fh hh}
\coeffentry{1284}{-\tfrac{1}{8400}}{ab ac cd fg}{eg fh gh hh}
\coeffentry{1285}{-\tfrac{1}{16800}}{ab ac cd fg}{eg fh hh hh}
\coeffentry{1286}{\tfrac{1}{12600}}{ab ac cd fg}{eh eh gh hh}
\coeffentry{1287}{-\tfrac{1}{12600}}{ab ac cd fg}{eh fh gh hh}
\coeffentry{1288}{\tfrac{1}{6300}}{ab ac de fh}{ef fg gg hh}
\coeffentry{1289}{-\tfrac{1}{12600}}{ab ac df ef}{dh eg gg hh}
\coeffentry{1290}{-\tfrac{1}{6300}}{ab ac dg ef}{eg fg fh hh}
\end{multicols}
\endgroup